\documentclass{statsoc}

\usepackage{amsthm}
\usepackage[a4paper]{geometry}
\usepackage{graphicx}
\usepackage[textwidth=8em,textsize=small]{todonotes}
\usepackage{amsmath,amsfonts,amssymb,bm}
\usepackage{natbib}
\usepackage{xspace}

\usepackage{xr-hyper}
\usepackage{hyperref}
\usepackage{enumitem}

\DeclareMathOperator*{\argmin}{arg\,min}

\DeclareMathOperator{\Cov}{\mathrm{Cov}}
\DeclareMathOperator{\vecs}{\mathrm{vec}}
\DeclareMathOperator{\prox}{\mathrm{prox}}
\newcommand{\inLaw}{\xrightarrow[]{\mathcal{L}}}

\definecolor{wjs}{RGB}{0,0,255}

\theoremstyle{plain}
\newtheorem{thm}{Theorem}[section]
\newtheorem{lem}{Lemma}[section]
\newtheorem{remark}{Remark}[section]
\newtheorem{proposition}{Proposition}[section]
\newtheorem{assumption}{Assumption}[section]

\newcommand{\Var}{\operatorname{\mathbb{V}\textnormal{ar}}}

\newcommand{\E}{\operatorname{\mathbb{E}}}

\newcommand{\x}{{\bm x}}

\renewcommand{\b}{\mathbf{b}}
\newcommand{\z}{\mathbf{z}}
\newcommand{\V}{\mathbf{V}}
\newcommand{\Sig}{\mathbf{\Sigma}}

\title[Moment-Adjusted Stochastic Gradients]{Statistical Inference for the Population Landscape via Moment-Adjusted Stochastic Gradients}
\author{Tengyuan Liang\thanks{Liang gratefully acknowledges support from the George C. Tiao Faculty Fellowship.}}
\address{University of Chicago, Booth School of Business, USA}
\email{tengyuan.liang@chicagobooth.edu}
\author[Liang and Su]{Weijie J.~Su\thanks{Su gratefully acknowledges support from NSF via grant CCF-1763314.}}
\address{University of Pennsylvania, Wharton School, USA}
\email{suw@wharton.upenn.edu}

\begin{document}

\begin{abstract}

Modern statistical inference tasks often require iterative optimization methods to compute the solution. Convergence analysis from an optimization viewpoint only informs us how well the solution is approximated numerically but overlooks the sampling nature of the data. In contrast, recognizing the randomness in the data, statisticians are keen to provide uncertainty quantification, or confidence, for the solution obtained using iterative optimization methods. This paper makes progress along this direction by introducing the moment-adjusted stochastic gradient descents, a new stochastic optimization method for statistical inference. We establish non-asymptotic theory that characterizes the statistical distribution for certain iterative methods with optimization guarantees. On the statistical
front, the theory allows for model mis-specification, with very mild conditions on the data. For optimization, the theory is flexible for both convex and non-convex cases. Remarkably, the moment-adjusting idea motivated from ``error standardization'' in statistics achieves a similar effect as acceleration in first-order optimization methods used to fit generalized linear models. We also demonstrate this acceleration effect in the non-convex setting through numerical experiments.

\end{abstract}


\section{Introduction}

Statisticians are interested in inferring properties about a population based on independently sampled data. In the parametric regime, the inference problem boils down to constructing point estimates and confidence intervals for a finite number of unknown parameters. When the data-generation process is well-specified by the parametric family, an elegant asymptotic theory --- credited to Ronald Fisher in the 1920s --- has been established for maximum likelihood estimation (MLE). This asymptotic theory is readily generalizable to the model mis-specification setting, for a properly chosen risk function $\ell(\theta, z)$ and the corresponding empirical risk minimizer (ERM)
\begin{align*}
	\widehat{\theta}_{\rm ERM} & \triangleq \argmin_{\theta} ~ \frac{1}{N} \sum_{i=1}^{N} \ell(\theta, z_i) &\text{empirical risk minimizer}, \\
	\theta_* &\triangleq \argmin_{\theta} \E_{\mathbf{z} \sim P} \ell(\theta, \mathbf{z})  &\text{population minimizer},
\end{align*}
with
\begin{align*}
	\sqrt{N} \left( \widehat{\theta}_{\rm ERM} - \theta_* \right) \inLaw \mathcal{N}\left(0, \mathbf{H}(\theta_*)^{-1} \mathbf{\Sigma}(\theta_*) \mathbf{H}(\theta_*)^{-1}\right).
\end{align*}
Here $\theta$ is the parameter of the model, $z_i$'s are i.i.d draws from an unknown distribution $P$, Hessian $\mathbf{H}(\theta) \triangleq \E \left[\nabla^2_\theta \ell(\theta, \mathbf{z})  \right]$, and $\mathbf{\Sigma}(\theta) \triangleq \E \left[ \nabla_\theta \ell(\theta, \mathbf{z}) \otimes \nabla_\theta \ell(\theta, \mathbf{z}) \right]$.
Define the \textit{population landscape} $L(\theta)$ as\footnote{It is also called loss function in the statistical learning literature. In generalized methods of moment, $\mathbb{E}_{\mathbf{z} \sim P} \nabla_\theta \ell(\theta, \mathbf{z}) = 0$ is also called moment condition. The MLE can be also viewed as a special case with $\ell(\theta, \mathbf{z}) = -\log p_{\theta}(\mathbf{z})$ and the data-generation process being $P = P_{\theta_*}$.}
\begin{align}
	L(\theta) \triangleq \E_{\mathbf{z} \sim P} \ell(\theta, \mathbf{z}).
\end{align}
One should notice that the elegant statistical theory for inference holds under rather mild regularity conditions, without requiring a convex $L(\theta)$. However, it overlooks one important aspect: the optimization difficulty of the landscape on $\theta$.

Optimization techniques are required to solve for the above estimator $\widehat{\theta}$, as they rarely take closed form. Global convergence and computational complexity is only well-understood when the sample analog $\frac{1}{N} \sum_{i=1}^N \ell(\theta, z_i)$ is convex. The optimization is done iteratively
\begin{align}
	\label{eq:updates}
\theta_{t+1} = \theta_t - \eta \mathbf{h}(\theta_t),
\end{align}
where the vector field $\mathbf{h}$ is based on the first- and/or second-order information, $\eta$ is step-size.
For the non-convex case, the convergence becomes less clear, but in practice people still employ these iterative methods. Nevertheless, in either case, the available convergence results fall short of the statistical goal: after a certain number of iterations, one is interested in knowing the sampling distribution of $\theta_t$, for uncertainty quantification of the optimization algorithm.

The goal of the present work is to combine the strength of the two worlds in inference and optimization: to characterize the statistical distribution of the iterative methods, with good optimization guarantee. Specifically,
we study particular stochastic optimization methods for the (possibly non-convex) population landscape $L(\theta)$ in the fixed dimension regime, and at the same time characterize the sampling distribution at each step, through establishing a non-asymptotic theory. We allow for model mis-specification, and require only mild moment conditions on the data-generating process.

\subsection{Motivation}

Observe the simple fact that what one actually wishes to optimize is the population objective $L(\theta) = \E_{\mathbf{z} \sim P} \ell(\theta, \mathbf{z})$, not the sample version. Therefore, stochastic approximation pioneered by \cite{robbins1951stochastic, kiefer1952stochastic} stands out as a natural optimization approach for the statistical inference problem.
In modern practice, \textit{Stochastic Gradient Descent} (SGD) with mini-batches of size $n$ is widely used,
\begin{align}
	\theta_{t+1} = \theta_{t} - \eta \widehat{\E}_n \nabla_\theta \ell(\theta_t, \mathbf{z}),
\end{align}
where $\widehat{\E}_n$ is the empirical expectation over $n$ independently sampled mini-batch data.

Our first observation follows from the intuition that Gaussian approximation holds for each step when $n$ is not too small, which we will make rigorous in a moment. Define
\begin{align}
	\mathbf{b}(\theta) &= \E_{\mathbf{z} \sim P} \nabla_{\theta} \ell(\theta, \mathbf{z}), \label{eq:grad} \\
	\mathbf{V}(\theta) &= \left\{ \Cov[\nabla_{\theta} \ell(\theta, \mathbf{z}) ] \right\}^{1/2}, \label{eq:cov}
\end{align}
then observe the following approximation for \eqref{eq:heuristic} via Central Limit Theorem (CLT)
\begin{align}
	\theta_{t+1} &= \theta_{t} - \eta \widehat{\E}_n \nabla_\theta \ell(\theta_t, \mathbf{z}) \nonumber \\
	&= \theta_{t} - \eta \E \nabla_\theta \ell(\theta_t, \mathbf{z}) + \eta \left[ \E \nabla_\theta \ell(\theta_t, \mathbf{z}) - \widehat{\E}_n \nabla_\theta \ell(\theta_t, \mathbf{z}) \right] \nonumber \\
	& \approx \theta_{t} - \eta \mathbf{b}(\theta_t) + \sqrt{2\beta^{-1} \eta} \mathbf{V}(\theta_t) \mathbf{g}_t \label{eq:heuristic}, \quad  \text{with}~\beta \triangleq \frac{2n}{\eta},
\end{align}
where $\mathbf{g}_t, t\geq 0$ are independent isotropic Gaussian vectors\footnote{CLT states that for $X_i, i\in [n]$ i.i.d  sampled, asymptotically the following convergence in distribution holds $\sqrt{n} \left[ \frac{1}{n} \sum_{i=1}^n X_i - \E X \right] \inLaw \mathcal{N}(0, \Cov(X)).$
	If we substitute $X_i = \V(\theta_t)^{-1} \nabla_\theta \ell(\theta_t, Z_i)$, condition on $\theta_t$, one can see where the isotropic Gaussian emerges.}.
The combination of $n,\eta$ provides a stronger approximation guarantee at each iteration
for large $n$, in contrast to the asymptotic normal approximation for the average trajectory in \cite{polyak1992acceleration} as $t\rightarrow \infty$.
The $\beta^{-1}$ quantifies the ``variance'' injected in each step (due to sampled mini-batches), or the ``temperature'' parameter: the larger the $\beta$ is, the closer the distribution is concentrated near the deterministic steepest gradient descent updates. The scaling of the step-size $\eta$ relates to Cauchy discretization of the It\^{o} diffusion process (as $\eta \rightarrow 0$) $$d \theta_t = - \mathbf{b}(\theta_t) dt + \sqrt{2\beta^{-1}} \mathbf{V}(\theta_t) dB_t.$$

Our second observation comes from a classic ``standardization'' idea in statistics --- we want to adjust the stochastic gradient vector at step $t$ by $\mathbf{V}(\theta_t)$ so that the conditional noise (conditioned on $\theta_t$) for each coordinate is independent and homogenous,
\begin{align}
	\theta_{t+1} &= \theta_{t} - \eta \mathbf{V}(\theta_t)^{-1} \widehat{\E}_n \nabla_\theta \ell(\theta_t, \mathbf{z}) \nonumber \\
	& \approx \theta_{t} - \eta \mathbf{V}(\theta_t)^{-1} \mathbf{b}(\theta_t) + \sqrt{2\beta^{-1} \eta} \mathbf{g}_t \label{eq:heuristic-masg}.
\end{align}
Namely, noisier gradient information is weighted less.
This standardization trick in statistics is similar to the Newton/quasi-Newton method in second-order optimization, though with notable difference. The similarity lies in the fact the noisy gradient information is weighted according to some local version of ``curvature.'' However, the former uses root of second moment matrix, while the latter uses Hessian (second-order derivatives).

To answer the inference question about $L(\theta)$ using the ``moment-adjusted'' iterative method proposed in \eqref{eq:heuristic-masg}, one needs to know the sampling distribution of $\theta_t$ for a fixed $t$. One hopes to directly describe the distribution in a non-asymptotic fashion, instead of characterizing this distribution either through the asymptotic normal limit \citep{polyak1992acceleration} (passing over data one at a time) in the convex senario, or through the invariant distribution which could in theory take exponential time to converge for general non-convex $L(\theta)$ \citep{bovier2004metastability,raginsky2017non}. One thing to notice is that, at a fixed time $t$, the distribution is distinct from Gaussian, for general $\mathbf{b}$ and $\mathbf{V}$.
From an optimization angle, one would like the iterative algorithm to converge (to a local optima) quickly. This is also important for the purpose of inference: given the distribution can be approximately characterized at each step, one hopes that the distribution will concentrate near a local minimum of the population landscape $L(\theta)$ within a reasonable time budget, before the error accumulates in the stochastic process and invalidates the approximation.

\paragraph{Notations}
For a vector $v$, $\| v \| = \sqrt{v^T v}$ denotes the $\ell_2$ norm, and $v \otimes v = vv^T$ denotes the outer-product. We use $\| M \|$ to denote the operator norm for a matrix $M$. For a positive semi-definite matrix $M$, $\langle v, w \rangle_{M} =  v^T M w$. We use $t \in [T]$ to denote indices $0\leq t\leq T$, and ``$\inLaw$'' for convergence in distribution. For two matrices $A$ and $B$, we use $A \otimes_K B$ to represent the Kronecker product. Moreover, $O, o$ are the Bachmann-Landau notations and $O_{\bf p}$ denotes stochastic boundedness. In the discussion, we use $O_{\epsilon, \delta}(\cdot)$ to denote the order of magnitude for parameters $\epsilon, \delta$ only, treating others as constants. For two probability measures $\mu, \nu$, we use $D_{\rm KL}(\mu, \nu)$ and $D_{\rm TV}(\mu, \nu)$ to denote the Kullback-Leibler and total variation distance respectively. Throughout, we denote the population gradient $\b \in \mathbb{R}^p$, moment matrix $\V, \Sig \in \mathbb{R}^{p \times p}$ using the boldface notation, with the hope of emphasizing their role in the paper.

\subsection{Contributions and Organization}

We propose the \textit{\textbf{M}oment-\textbf{a}djusted \textbf{s}tochastic \textbf{Grad}ient descent} (MasGrad), an iterative optimization method that infers the stationary points of the population landscape $L(\theta)$, namely $\{ \theta \in \mathbb{R}^p: \|\nabla L(\theta) \| = 0 \}$. The MasGrad is a simple variant of SGD that adjusts the descent direction using $\mathbf{V}(\theta_t)^{-1}$ (defined in \eqref{eq:cov}, the square root of the inverse covariance matrix) at the current location,
\begin{align*}
	\text{MasGrad}: \quad \theta_{t+1} &= \theta_{t} - \eta \mathbf{V}(\theta_{t})^{-1} \widehat{\E}_n \nabla_\theta \ell(\theta_t, \mathbf{z}).
\end{align*}

We summarize our main contributions in two perspectives. Extensions including estimation and computation of the moment-adjusted gradients will be discussed later in Section~\ref{sec:est-comp-mat-root}. 

\smallskip
\noindent \textbf{Inference.} \quad The distribution of MasGrad updates $\theta_{t} \in \mathbb{R}^p$, with $n$ independently sampled mini-batch data at each step, can be characterized in a non-asymptotic fashion. Informally, for any data-generating distribution $\mathbf{z} \sim P$ under mild conditions, the distribution of $\theta_t$ --- denoted as $\mu(\theta_{t})$ --- satisfies,
\begin{align*}
	D_{\rm TV} (\mu(\theta_t), \nu_{t, \eta}) \leq O_{t, n} \left( \sqrt{\frac{t}{n}} \right)  \quad \Rightarrow \quad \mu(\theta_t) \inLaw \nu_{t, \eta}, ~\text{converge in distribution as $n \rightarrow \infty$}.
\end{align*}
Here $\nu_{t, \eta}$ is the distribution of $\xi_t$ that follows the update initialized with $\xi_0 = \theta_0$
\begin{align}
	\label{eq:xi}
 \xi_{t+1} = \xi_t - \eta \mathbf{V}(\xi_t)^{-1} \mathbf{b}(\xi_t) + \sqrt{2\beta^{-1} \eta} \mathbf{g}_t, ~ \mathbf{g}_t \sim \mathcal{N}(0, I_p)~\text{and}~\beta = \frac{2n}{\eta}.
\end{align}
Remark that $\nu_{t, \eta}$ only depends on $t, \eta$, and the first and second moments $\mathbf{b},\mathbf{V}$ of $\nabla \ell(\theta, \mathbf{z})$, regardless of the specific data-generating distribution $\mathbf{z} \sim P$.
The rigorous statement is deferred to Thm.~\ref{thm:couple.p}, and further extensions to the continuous time analog are discussed in Appendix~\ref{sec:continuous-langevin}.

\smallskip
\noindent \textbf{Optimization.} \quad Interestingly, in the strongly convex case such as in generalized linear models (GLMs), the ``standardization'' idea achieves the Nesterov acceleration \citep{nesterov1983method, nesterov2013introductory}. Informally, the number of iterations for an $\epsilon$-minimizer for gradient descent requires
\begin{align*}
	T_{\rm GD} = O_{\epsilon, \kappa}\left( \kappa \log \frac{1}{\epsilon} \right), \quad \text{for some $\kappa > 1$}.
\end{align*}
We show that for GLMs under mild conditions, MasGrad reduces the number of iterations to
\begin{align*}
	T_{\rm MasGrad} = O_{\epsilon, \kappa}\left( \sqrt{\kappa} \log \frac{1}{\epsilon} \right),
\end{align*}
which matches Nesterov's acceleration in the strongly convex case. The formal statement is deferred to Section~\ref{sec:acceleration}, where extensions including proximal updates are discussed.

\smallskip
Combining the inference and optimization theory together, we present informally the results for both the \textit{convex} and \textit{non-convex} cases. Recall that $\theta \in \mathbb{R}^p$.

\smallskip
\noindent \textbf{Convex.} \quad In the strongly convex case, MasGrad with a properly chosen step-size and the following choice of parameters
\begin{align*}
	T = O_{\epsilon}\left(\log \frac{1}{\epsilon}\right) ~~\text{and}~~ n = O_{\epsilon,p}\left( \frac{p}{\epsilon} \right),
\end{align*}
satisfies
	\begin{align*}
		&\text{inference}: \quad D_{\rm TV}\left( \mu(\theta_T), \mu(\xi_T) \right) \leq O_{\epsilon}\left( \sqrt{\epsilon \log 1/\epsilon} \right),\\
		&\text{optimization}: \quad \E L(\theta_T) - \min_{\theta} L(\theta) \leq \epsilon, ~\E L(\xi_T) - \min_{\theta} L(\theta) \leq \epsilon,~~\text{where $\xi_T \sim \nu_{T, \eta}$,}
	\end{align*}
	where the evolution of $\xi_t$ is defined in \eqref{eq:xi}.
Here the total number of samples needed is $nT = O_\epsilon(\epsilon^{-1} \log 1/\epsilon)$. The formal result is stated in Thm.~\ref{thm:converge}.

\smallskip
\noindent \textbf{Non-convex.} \quad Under mild smoothness conditions, MasGrad with a proper step-size and the following choice of parameters
\begin{align*}
	T = O_{\epsilon,\delta, p}\left( \frac{1 \vee p\delta^2}{\epsilon^2}\right) ~~\text{and}~~ n = O_{\epsilon,\delta, p}\left( \frac{\delta^{-2} \vee p}{\epsilon^2}\right),
\end{align*}
satisfies
	\begin{align*}
		&\text{inference}: \quad D_{\rm TV}\left( \mu(\theta_t, t\in [T]), \mu(\xi_t, t\in [T]) \right) \leq O_{\delta}(\delta),\\
		&\text{optimization}: \quad \E \min_{t \leq T} \| \nabla L(\theta_t) \| \leq \epsilon, ~\E \min_{t \leq T} \| \nabla L(\xi_t) \| \leq \epsilon,~~\text{where $\xi_t \sim \nu_{t,\eta}$, for $t \in [T]$.}
	\end{align*}
	Here the total number of samples needed is $nT = O_{\epsilon,\delta}(\epsilon^{-4} \delta^{-2})$. The formal result is deferred to Thm.~\ref{thm:non-convex}.

\section{Relations to the Literature}
\label{sec:literature}

In the case of a differentiable convex $L(\theta)$, finding a minimum is equivalent to solving $\nabla L(\theta) = 0$. This simple equivalence reveals that the vanilla SGD, which takes the form\footnote{Recognize that $\nabla_\theta \ell(\theta_t, z_t)$ is an unbiased estimate of the population gradient as $\nabla_\theta L(\theta_t) = \E_{\mathbf{z} \sim P}[\nabla_\theta \ell(\theta_t, \mathbf{z})]$.}
\begin{align}
	\label{eq:sto.approx}
	\theta_{t+1} = \theta_{t} - \eta_t \nabla_\theta \ell(\theta_t, z_t),
\end{align}
is an instance of stochastic first-order approximation methods. This class of methods are iterative algorithms that attempt to solve fixed-point equations (for example, $\nabla L(\theta) = 0$) provided noisy observations (for example, $\nabla_\theta \ell(\theta_t, z_t)$) \citep{robbins1951stochastic,kiefer1952stochastic,toulis2017asymptotic,chen2016statistical,li2017statistical}. Using slowly diminishing step-sizes $\eta_t = O(1/t^{\alpha})$ ($\alpha<1$), \citet{ruppert1988} and \citet{polyak1990} showed that acceleration using the average over trajectories of this recursive stochastic approximation algorithm attains optimal convergence rate for a strongly convex $L$ (see \cite{polyak1992acceleration} for more details). Recently, the running time of stochastic first-order methods are considerably improved using combinations of variance-reduction techniques \citep{roux2012stochastic,johnson2013accelerating} and Nesterov's acceleration \citep{ghadimi2016accelerated,cotter2011better,jofre2017variance,ghadimi2012optimal,arjevani2016oracle}.

Despite the celebrated success of stochastic first-order methods in modern machine learning tasks, researchers have kept improving the per-iteration complexity of second-order methods such as Newton or quasi-Newton methods, due to their faster convergence. 
A fruitful line of research has focused on how to improve asymptotic convergence rate as $t \rightarrow \infty$ through pre-conditioning, a technique that involves approximating the unknown Hessian $\mathbf{H}(\theta) = \nabla^2_\theta L(\theta)$ (see, for instance, \citet{bordes2009sgd} and references therein).
Utilizing the curvature information reflected by various efficient approximations of the Hessian matrix, stochastic quasi-Newton methods \citep{moritz2016linearly,byrd2016stochastic,wang2017stochastic,schraudolph2007stochastic,mokhtari2015global,becker2012quasi}, Newton sketching or subsampled Newton \citep{pilanci2015newton,xu2016sub,berahas2017investigation,bollapragada2016exact}, and stochastic approximation of the inverse Hessian via Taylor expansion \citep{agarwal2017second} have been proposed to strike balance between convergence rate and per-iteration complexity.

In the information geometry literature, one closely related method is the natural gradient \citep{amari1998natural,amari2012differential}. When the parameter space enjoys a certain structure, it has been shown that natural gradient outperforms the classic gradient descent both theoretically and empirically. To adapt the natural gradient to our setting, we relate the loss function to a generative model $\ell(\theta, z) = - \log p_{\theta}(z)$. The Riemannian structure of the parameter space (manifold) of the statistical model is defined by the Fisher information
\begin{align*}
	\mathbf{I}(\theta) = \E_{\mathbf{z} \sim P} \left[ \nabla_{\theta} \ell(\theta, \mathbf{z}) \otimes \nabla_{\theta} \ell(\theta, \mathbf{z})  \right].
\end{align*}
The natural gradient can be viewed as the steepest descent induced by the Riemannian metric
\begin{align*}
	\theta_{t+1} &= \argmin_{\theta} \left[ L(\theta_t) + \langle \nabla_{\theta} L(\theta_t), \theta - \theta_t \rangle + \frac{1}{2\eta_t} \| \theta - \theta_t \|_{\mathbf{I}(\theta_t)}^2 \right] \\
	&= \theta_t - \eta_t \mathbf{I}(\theta_t)^{-1} \nabla_{\theta} L(\theta_t).
\end{align*}
Note the intimate connection between natural gradient descent and approximate second-order optimization method, as the Fisher information can be heuristically viewed as an approximation of the Hessian \citep{schraudolph2002fast, martens2014new}.

Another popular and closely related example as such is AdaGrad \citep{duchi2011adaptive}, which is a variant of SGD that adaptively determines learning rates for different coordinates by incorporating the geometric information of past iterates. In its simplest form, AdaGrad records previous gradient information through
\[
G_t = \sum_{i=1}^t \nabla \ell(\theta_i, z_i) \otimes \nabla \ell(\theta_i, z_i),
\]
and this procedure then updates iterates according to
\[
\theta_{t+1} = \theta_t - \gamma G^{-\frac12}_t \nabla \ell(\theta_t, z_t),
\]
where $\gamma > 0$ is fixed. In large-scale learning tasks, evaluating $G_t^{-\frac12}$ is computationally prohibitive and thus is often suggested to use $\text{diag}(G_t)^{-\frac12}$ instead. It should be noted, however, that the theoretical derivation of regret bound for AdaGrad considers $G_t^{-\frac12}$. AdaGrad is a flexible improvement on SGD and can easily extend to non-smooth optimization and non-Euclidean optimization such as mirror descent. With the geometric structure $G_t$ learned from past gradients, AdaGrad assigns different learning rates to different components of the parameter, allowing infrequent features to take relatively larger learning rates. This adjustment is shown to speed up convergence dramatically in a wide range of empirical problems \citep{pennington2014glove}.

Stochastic Gradient Langevin Dynamics (SGLD) has been an active research field in sampling and optimization in recent years \citep{welling2011bayesian, dalalyan2017theoretical,
bubeck2015sampling, raginsky2017non, mandt2017stochastic, brosse2017sampling, tzen2018local, durmus2018efficient}. SGLD injects an additional $\sqrt{2\beta^{-1}\eta}$ level isotropic Gaussian noise to each step of SGD with step-size $\eta$, where $\beta$ is the inverse temperature parameter. Besides similar optimization benefits as SGD such as convergence and chances of escaping stationary points, the injected randomness of SGLD provides an efficient way of sampling from the targeted invariant distribution of the continuous-time diffusion process, which has been shown to be useful statistically in Bayesian sampling \citep{welling2011bayesian, mandt2017stochastic, durmus2018efficient}. 

In the current paper, we take a distinct approach: we motivate and analyze a variant of SGD through the lens of Langevin dynamics, from a frequentist point of view, and then present the optimization benefits as a by-product of the statistical motivation. The approximation in Eqn.~\eqref{eq:heuristic} relates the density evolution of $\theta_s$ to a discretized version of It\^{o} diffusion process (as $\eta \rightarrow 0$)
\begin{align*}
   d \theta_s = - \mathbf{b}(\theta_s) ds + \sqrt{2\beta^{-1}} \mathbf{V}(\theta_s) dB_s.
\end{align*}

The invariant distribution $\pi(\theta)$ satisfies the following Fokker--Planck equation
\begin{align*}
	\beta^{-1}  \sum_{i,j} \frac{\partial^2}{\partial x_i x_j} (\pi \mathbf{a}_{ij}) + \sum_{i} \frac{\partial}{\partial x_i} (\pi \mathbf{b}_{i}) = 0
\end{align*}
where $\mathbf{a}_{ij}(x) = (\mathbf{V}(x) \mathbf{V}(x)')_{ij}$.
In general, the stationary distribution is hard to characterize unless both $\mathbf{V}$ and $\mathbf{b}$ take special simple forms.
For example, when $\mathbf{b}(x)$ is linear and $\mathbf{V}(x)$ is independent of $x$ as in \citep{mandt2017stochastic}, the diffusion process reduces to Ornstein-Uhlenbeck process with multivariate Gaussian as the invariant distribution. Another simple case is when $\mathbf{V}(x) = \mathbf{I}$, the diffusion process is also referred to as Langevin dynamics, with the Gibbs measure $\pi(\theta) \propto \exp(-\beta L(\theta))$ as the unique invariant distribution \citep{welling2011bayesian, dalalyan2017theoretical, raginsky2017non}.

\section{Statistical Inference via Langevin Diffusion}
\label{sec:stat}
In this section we will explain why \textit{\textbf{M}oment-\textbf{a}djusted \textbf{s}tochastic \textbf{Grad}ient descent} (MasGrad) produces recursive updates whose statistical distribution can be characterized. We would like to mention that MasGrad at the same time achieves significant acceleration in optimization in the strongly convex case (detailed in Section~\ref{sec:acceleration}). For the general non-convex case, we provide non-asymptotic theory for inference and optimization in Section~\ref{sec:non-convex}. We first present the simplest version of the algorithm, assuming that $\mathbf{V}(\theta)^{-1}$ can be evaluated at any given $\theta$. Statistical estimation and efficient direct computation of $\mathbf{V}(\theta)^{-1}$ will be discussed in Section~\ref{sec:est-comp-mat-root}.

Recall the MasGrad we introduced, which adjusts the gradient direction using the root of the inverse covariance matrix at the current location,
\begin{align}
	\label{eq: MasGrad}
	\text{MasGrad}: \quad \theta_{t+1} &= \theta_{t} - \eta \mathbf{V}(\theta_{t})^{-1} \widehat{\E}_n \nabla_\theta \ell(\theta_t, \mathbf{z}).
\end{align}
As we have heuristically outlined in Eqn.~\eqref{eq:heuristic}, the MasGrad can be approximated by the following discretized Langevin diffusion,
\begin{align}
	\label{eq:discrete.diff}
	\text{Discretized diffusion}: \quad \xi_{t+1} &= \xi_{t} - \eta \mathbf{V}(\xi_{t} )^{-1} \mathbf{b}(\xi_{t}) + \sqrt{2\beta^{-1} \eta} \mathbf{g}_t.
\end{align}
In this section, we establish non-asymptotic bounds on the distance between the distribution of MasGrad process $\mathcal{L}(\theta_t, t \in [T])$ and discretized diffusion process $\mathcal{L}(\xi_t, t \in [T])$.

The proof is based on the entropic Central Limit Theorem (entropic CLT) \citep{barron1986entropy, bobkov2013, bobkov2014berry}. The classic CLT based on convergence in distribution is too weak for our purpose: we need to translate the non-asymptotic bounds at each step to the whole stochastic process. It turns out that the entropic CLT couples naturally with the chain-rule property of relative entropy, which together provides non-asymptotic characterization on closeness of the distributions for the stochastic processes.

Let's first state the standard assumptions for entropic CLT. These assumptions can be found in \citep{bobkov2013}. Remark that we are focusing on fixed dimension setting.  
\begin{enumerate}[label={\bf (A.\arabic*)}]
	\item Absolute continuity to Gaussian: assume random vector $X \in \mathbb{R}^p$ has bounded entropic distance to the Gaussian distribution, for some constant $D_1$ 
	\begin{align*}
		D_{\rm KL}\left( \mu(X) ||  \mu(\mathbf{g}) \right) < D_1, \quad \text{where $\mathbf{g} \sim \mathcal{N}(0,I_p)$.}
	\end{align*}
	\item Finite $(4+\delta)$-th moments: assume that there exists constant $D_2$
	\begin{align*}
		\mathbb{E} \| X \|^{4+\delta} < D_2, \quad \text{for some small $\delta > 0$.}
	\end{align*}
\end{enumerate}

Define $\forall i$, the stochastic component of the adjusted gradient direction
\begin{align}
	\label{eq:asmp}
X_i(\theta) = \mathbf{V}(\theta)^{-1} \left[ \nabla_\theta \ell(\theta, z_i) -  \E_{\mathbf{z} \sim P} \nabla_\theta \ell(\theta, \mathbf{z}) \right].
\end{align}
It is clear that $X_i$'s are i.i.d. with $\E X_i(\theta) = 0$ and $\Cov [X_i(\theta)] = I_p$. Here $X_i(\theta)$ is defined on the same $\sigma$-field as $z_i$ drawn from $P$.

\begin{thm}[Non-asymptotic bound for inference]
	\label{thm:couple.p}
	Let $\mu(\theta_t, t \in [T])$ denote $\mathcal{L}(\theta_t, t \in [T])$, the joint distribution of MasGrad process, and $\mu( \xi_t, t \in [T] )$ be the joint distribution of the discretized diffusion process in \eqref{eq:discrete.diff}. Consider the same initialization $\theta_0 = \xi_0$. 
	
	Assume that uniformly for any $\theta$, $X(\theta)$ defined in \eqref{eq:asmp} satisfies {\bf(A.1)} and {\bf(A.2)} with constants $D_1, D_2$ that only depends on $p$. Then the following bound holds,
	\begin{align}
		D_{\rm TV}\left( \mu(\theta_t, t \in [T]), \mu( \xi_t, t \in [T] ) \right) \leq  C \sqrt{\frac{T}{n} + o\left( \frac{T(\log n)^{\frac{p - (4+\delta)}{2}}}{n^{1+\frac{\delta}{2}}} \right) },
	\end{align}
	where $C$ is some constant that depends on the $D_1$ and $D_2$ only.
\end{thm}

\begin{remark}
	\rm
	The above theorem characterizes the sampling distribution of MasGrad -- $\theta_t$, using a measure that only depends on the first and second moments of $\nabla \ell(\theta, \mathbf{z})$, namely $\mathbf{V}(\theta)^{-1}\mathbf{b}(\theta)$, regardless of the specific the data-generating distribution $\mathbf{z} \sim P$. Observe that the distribution closeness is established in a strong total variation distance sense, for the two stochastic processes $\{\theta_t, t\in [T]\}$ and $\{ \xi_t, t \in [T]\}$. If we dig in to the proof, one can easily obtain the following marginal result
	\begin{align*}
		D_{\rm TV}\left( \mu(\theta_T), \mu(\xi_T) \right) \leq \sqrt{2 D_{\rm KL} \left( \mu(\theta_T) || \mu(\xi_T) \right) } \leq \sqrt{2 D_{\rm KL} \left( \mu(\theta_t, t\in [T]) || \mu(\xi_t, t\in [T]) \right) },
	\end{align*}
where the last inequality follows from the chain-rule of relative entropy. 
Therefore, one can as well prove for the last step distribution
	$$
	D_{\rm TV}\left( \mu(\theta_T), \mu(\xi_T) \right) \leq C \sqrt{\frac{T}{n} }.
	$$
	
\end{remark}

\begin{remark}
	\rm
	
	One important fact about Thm.~\ref{thm:couple.p} is that it holds for any step-size $\eta$, which provides us additional freedom of choosing the optimal step-size for the optimization purpose. This theorem is stated in the fixed dimensional setting when $p$ does not change with $n$. Remark in addition that the Gaussian approximation at each step still holds with high probability, in the moderate dimensional setting when $p = o(\frac{\log n}{\log \log n})$, as shown in the non-asymptotic bound in the above Thm.~\ref{thm:couple.p}. We would like to emphasize that the current paper only considers the fixed dimension setting, while considering the mini-batch sample size $n$ and running time $T$ varying. Assumptions (A.1) and (A.2) are standard assumptions in entropic CLT: (A.1) states that the distribution for each stochastic gradient is non-lattice with bounded relative entropy to Gaussian; (A.2) is the standard weak moment condition. Note here that the constants $D_1$ and $D_2$ depend on the dimension implicitly.

	For the purpose of statistical inference, one can always approximately characterize the distribution of MasGrad using Thm.~\ref{thm:couple.p}. As an additional benefit, the result naturally provides us an algorithmic way of sampling this target universal distribution $\mu(\xi_t)$.
	For some particular tasks, it remains of theoretical interest to analytically characterize the distribution of MasGrad using the continuous time Langevin diffusion and its invariant distribution. We defer the analysis of the discrepancy between the discretized diffusion to the continuous time analog to Appendix~\ref{sec:continuous-langevin}.
\end{remark}

\section{Convexity and Acceleration}
\label{sec:acceleration}

In this section, we will demonstrate that the ``moment-adjusting'' idea motivated from standardizing the error from an inference perspective achieves similar effect as acceleration in convex optimization. We will investigate \textit{Generalized Linear Models} (GLMs) as the main example. Later, we will also discuss the case with non-smooth regularization. It should be noted that using first-order information to achieve acceleration was first established in the seminal work by \cite{nesterov1983method, nesterov2013introductory} based on the ingenious notion of estimating sequence. Before diving into the technical analysis, we would like to point out that in MasGrad the moment-adjusting matrix $\V(\theta)$ can be estimated using only first-order information, however, as one will see, MasGrad achieves acceleration for GLMs in a way resembles the approximate second-order method such as quasi-Newton.

\subsection{Inference and optimization for optima}

Now we are ready to state the theory for inference and optimization using MasGrad in the strongly convex case. Let $L(w): \mathbb{R}^p \rightarrow \mathbb{R}$ be a smooth convex function. Recall $\mathbf{b}(w) = \nabla L(w)$, $\mathbf{H}(w) = \nabla^2 L(w)$ and $\mathbf{V}(w) \in \mathbb{R}^{p \times p}$ are positive definite matrices.
	Define
	\begin{align}
		\label{eq:conv.glm}
		\alpha &\triangleq \min_{v, w}~ \lambda_{\min} \left( \mathbf{V}(w)^{-1/2} \mathbf{H}(v) \mathbf{V}(w)^{-1/2}  \right) >0, \nonumber \\
		\gamma &\triangleq \max_{v, w}~ \lambda_{\max} \left( \mathbf{V}(w)^{-1/2} \mathbf{H}(v) \mathbf{V}(w)^{-1/2}  \right) >0. 
	\end{align}

\begin{thm}[MasGrad: strongly convex]
	\label{thm:converge}
    Let $\alpha, \gamma$ be defined as in \eqref{eq:conv.glm}.
	Consider the MasGrad updates $\theta_t$ in \eqref{eq: MasGrad} with step-size $\eta = 1/\gamma$, and the corresponding discretized diffusion $\xi_t$,
	\begin{align*}
		\xi_{t+1} = \xi_t - \eta \mathbf{V}(\xi_t)^{-1} \mathbf{b}(\xi_t)+ \sqrt{2\beta^{-1} \eta} \mathbf{g}_t,\quad \text{where $\beta = \frac{2n}{\eta}$}.
	\end{align*}
	Then for any precision $\epsilon >0$, one can choose
	\begin{align}
		\label{eq:conv-T-n}
		T = \frac{\gamma}{\alpha} \log \frac{2(L(\theta_0) - \min_{\theta} L(\theta))}{\epsilon} ~~\text{and}~~ n = \frac{4p \max_{\theta} \| \mathbf{V}(\theta) \|}{\alpha \epsilon},
	\end{align}
	such that
	\begin{align*}
		&(1)\quad D_{\rm TV}\left( \mu(\theta_t, t \in [T]), \mu( \xi_t, t \in [T] ) \right) \leq O_{\epsilon}\left( \sqrt{\epsilon \log (1/\epsilon)} \right),\\
		&(2)\quad \E L(\theta_t) - \min_{\theta} L(\theta) \leq \epsilon, ~ \E L(\xi_t) - \min_{\theta} L(\theta) \leq \epsilon,
	\end{align*}
	with in total $O_{\epsilon}(\epsilon^{-1} \log 1/\epsilon)$ independent data samples.
\end{thm}

\begin{remark}
	\rm
	In plain language, discretized diffusion process $\xi_t, t \in [T]$, whose distribution only depends on the adjusted moments $\V^{-1}\b$, approximates the sampling distribution of MasGrad $\theta_t, t \in [T]$ in a strong sense, i.e., the distribution of paths are close in TV distance. In addition, as a stochastic optimization method, MasGrad's optimization guarantee depends on the ``modified'' condition number defined in \eqref{eq:conv.glm}. Let's sketch the proof. Using Lemma~\ref{lem:conv.glm} in Appendix~\ref{sec:proof}, for all $t>0$, one can prove
	\begin{align*}
		\E L(\xi_{t}) - \min_{\theta} L(\theta) \leq \left(1 - \frac{\alpha}{\gamma}\right)^t (L(\theta_0) - \min_{\theta} L(\theta)) +  \max_{\theta} \| \mathbf{V}(\theta) \| \cdot \frac{\gamma}{\alpha} \beta^{-1} p.
	\end{align*}
	Therefore we can define the condition number of MasGrad as
		\begin{align}
			\label{eq:cond.num}
			\kappa_{\rm MasGrad} = \frac{\max_{w, v} \lambda_{\max} \left( [\mathbf{V}(w)]^{-1/2} \mathbf{H}(v) [\mathbf{V}(w)]^{-1/2} \right)}{\min_{w, v} \lambda_{\min} \left( [\mathbf{V}(w)]^{-1/2} \mathbf{H}(v) [\mathbf{V}(w)]^{-1/2} \right) },~~ \kappa_{\rm GD} = \frac{\max_{v} \lambda_{\max} \left(  \mathbf{H}(v)  \right)}{\min_{v} \lambda_{\min} \left( \mathbf{H}(v) \right)},
		\end{align}
	in contrast to the condition number in gradient descent.
	
	If $\beta = \frac{2n}{\eta}$ and $T, n$ are chosen as in \eqref{eq:conv-T-n},
	we know that $\E L(\xi_{T}) - L(\theta_*) \leq \epsilon$. Recall the result we establish in Thm.~\ref{thm:couple.p}, the total variation distance between MasGrad and the discretize diffusion in this case is bounded by
		$\sqrt{T/n} = O_{\epsilon} \left( \sqrt{\epsilon \log (1/\epsilon) } \right),$
	and the total number of samples used is of the order $nT = O_{\epsilon, p}(p/\epsilon \log (1/\epsilon))$. This result can be contrasted with the classical asymptotic normality for MLE or ERM: to achieve an $\epsilon$-minimizer,
	\begin{align*}
		\epsilon \geq L(\widehat{\theta}_N) - L(\theta_*) \asymp \| \widehat{\theta}_N - \theta_* \|^2 \asymp \frac{p}{N}  \Leftrightarrow N = O_{\epsilon,p}(p/\epsilon),
	\end{align*}
	the asymptotic sample complexity scales $O_{\epsilon, p}(p/\epsilon)$. Similar calculations also hold with the Ruppert--Polyak average on stochastic approximation with a carefully chosen decreasing step-size. As we can see, our result holds non-asymptotically, and it achieves both the optimization and inference goal, with an additional logarithmic factor.
\end{remark}

\subsection{Acceleration for GLMs}
Now let's take GLMs as an example to articulate the effect of acceleration. We will first use an illustrating toy example to show the intuition in an informal way, and then present the rigorous acceleration result for GLMs.
\smallskip

\paragraph{Toy example (informal).} Consider $y_i = \langle x_i, \theta_\ast \rangle + \epsilon_i$, $\epsilon_i \sim \mathcal{N}(0, \sigma^2)$ i.i.d. for $i \in [N]$. Let's focus on the fixed design case (where the expectation is only over $\mathbf{y}$), the loss $\ell(\theta, (x, y)) = \frac{1}{2}(\langle x, \theta \rangle - y)^2$. Denote $X \in \mathbb{R}^{N \times p}$, then we have
\begin{align*}
\mathbf{b}(\theta) &= \mathbb{E} \left[ \frac{1}{N} \sum_{i=1}^N ( x_i^T \theta - y_i) x_i \right] = \frac{1}{N} \sum_{i=1}^N x_i x_i^T (\theta - \theta_*) = \frac{1}{N} X^T X (\theta - \theta_*), \\
\mathbf{V}(w) &= \left[ \frac{1}{N} \sum_{i=1}^N x_i x_i^T \sigma^2 \right]^{1/2} = \sigma  \left[\frac{1}{N} X^T X\right]^{1/2},
\end{align*}
and the Hessian is $\mathbf{H}(w) = X^T X/N.$
Therefore, in this case, we have $$\kappa_{\rm MasGrad} = \sqrt{\kappa_{\rm GD}}.$$
By applying Lemma~\ref{lem:conv.glm} in Appendix~\ref{sec:proof}, one achieves the same effect as Nesterov's acceleration in the strongly convex case \citep{nesterov2013introductory}. Remark that the above analysis is to demonstrate the intuition, and is not rigorous --- as MasGrad is sensible with the random design.

\paragraph{Generalized linear models, random design, mis-specified model.} 

	Now let's provide a rigorous and unified treatment for the generalized linear models.
Consider the generalized linear model \citep{mccullagh1984generalized} where the response random variable $\mathbf{y}$ follows from the exponential family parametrize by $(\theta, \phi)$,
\begin{align*}
	f(y; \theta,\phi) = b(y, \phi) e^{\frac{y\theta - c(\theta)}{d(\phi)}}
\end{align*}
where $\mu = \mathbb{E} [\mathbf{y} | \mathbf{x} = x] = c'(\theta)$, $c''(\theta)>0$, and the natural parameter satisfies the linear relationship $\theta = \theta(\mu) = x^Tw$. In this case, we choose the loss function according to the negative log-likelihood
\begin{align*}
	\ell(w, (x, y)) = - y_i x_i^T w + c(x_i^T w).
\end{align*}
Special cases include, 
\begin{itemize}
	\item Bernoulli model (Logistic regression): $c(\theta) = \log (1 + e^\theta), ~\text{where}~x_i^Tw = \theta =\log \frac{\mu}{1 - \mu}$;
	\item Poisson model (Poisson regression): $c(\theta) = e^\theta, ~\text{where}~x_i^Tw = \theta= \log \mu$;
	\item Gaussian model (linear regression): $c(\theta) = \frac{1}{2}\theta^2, ~\text{where}~x_i^Tw = \theta = \mu$.
\end{itemize}

We are interested in inference even when the model can be \textit{mis-specified}.
Consider the statistical learning setting where $z_i = (x_i, y_i) \sim P = P_{\mathbf{x}} \times P_{\mathbf{y}|\mathbf{x}}, i\in [N]$ i.i.d. from some unknown joint distribution $P$. We are trying to infer the parameters $w$ by fitting the data using a parametric exponential family, however, we allow the flexibility that the exponential family model for $P(\mathbf{y}|\mathbf{x}=x)$ can be mis-specified. Specifically, the true regression function $m_*(x) = \mathbb{E}(\mathbf{y}|\mathbf{x} = x)$ may not be $c'(x^T w)$ for all $w$, namely, may not be realized by any model in the exponential family model class.
We have the population landscape
\begin{align}
	\label{eq:glm.landscape}
	L(w) = \E_{(\mathbf{x}, \mathbf{y})\sim P} \left[ - \mathbf{y} \mathbf{x}^T w + c(\mathbf{x}^T w) \right].
\end{align}

Define the conditional variance $\xi(x) = \Var(\mathbf{y}|\mathbf{x}=x) \in \mathbb{R}$ and the bias $\beta(\mathbf{x}, w) \triangleq c'(\mathbf{x}^T w) - m_*(\mathbf{x}) \in \mathbb{R}$, we have the following acceleration result for GLMs.

\begin{thm}[Acceleration]
	\label{thm:acceleration}
	Consider the condition number defined in ~\eqref{eq:cond.num} for MasGrad and GD, and assume that there exists constant $C>1$ such that for any $x, w, v$,
	\begin{align*}
		0 < \max \left\{ \frac{\xi(x)^2 + \beta(x, w)^2}{c''(x^T v)}, \frac{c''(x^T v)}{\xi(x)^2}  \right\} < C^{1/3}.
	\end{align*}
	Then for the optimization problem associated with GLMs defined in \eqref{eq:glm.landscape}, the following holds
	$$
	\kappa_{\rm MasGrad} < C \sqrt{\kappa_{\rm GD}}.
	$$
\end{thm}
\begin{remark}
	\rm
	The above theorem together with Lemma~\ref{lem:conv.glm} in Appendix~\ref{sec:proof} states that in the noiseless setting, the time complexity for MasGrad is
	$O\left( \sqrt{\kappa_{\rm GD}} \log 1/\epsilon \right)$
	in contrast to the complexity of GD --
	$O\left( \kappa_{\rm GD} \log 1/\epsilon \right)$, which is crucial when the condition number is large. The proof is based on matrix inequalities and the following analytic expressions,
	\begin{align*}
		&\mathbf{b}(w) = \mathbb{E}\left[ -\mathbf{y} \mathbf{x} + c'(\mathbf{x}^T w) \mathbf{x} \right] = \mathbb{E}\left[ (c'(\mathbf{x}^T w) - m_*(\mathbf{x})) \mathbf{x} \right], \\
		&\mathbf{V}(w) =  \left( \mathbb{E}[ \xi(\mathbf{x})^2 \mathbf{x} \mathbf{x}^T] + \Cov[\beta(\mathbf{x}, w) \mathbf{x} ]  \right)^{1/2}, \quad \mathbf{H}(w) =  \E\left[ c''(\mathbf{x}^T w) \mathbf{x} \mathbf{x}^T  \right].
\end{align*}
\end{remark}

\subsection{Non-smooth regularization}
In this section, we extend the acceleration result to problems with non-smooth regularization.
The main results are based on a simple modification called \textit{\textbf{M}oment-\textbf{ad}justed \textbf{Prox}imal Gradient descent} (MadProx). 

Consider the population loss function that can be decomposed into
\begin{align}
	\label{eq:non-smooth-decom}
	L(w) = g(w) + h(w)
\end{align}
where $g(w)$ is a smooth and convex function in $w$, and $h(w)$ is a non-smooth regularizer that is convex. Special cases include,
\begin{itemize}
	\item sparse regression with $\ell(w, (x_i, y_i)) = \frac{1}{2} (x_i^T w - y_i)^2 + \lambda \| w \|_1$ and
\begin{align*}
	L(w) = \E_{(\mathbf{x}, \mathbf{y}) \sim P} \left[\frac{1}{2} (\mathbf{x}^T w - \mathbf{y})^2 \right]  + \lambda \| w \|_1 := g(w) + h(w);
\end{align*}
	\item low rank matrix trace regression with $\ell(W, (X_i, y_i)) = \frac{1}{2} (\langle X_i, W \rangle - y_i)^2 + \lambda \| W \|_*$ 
\begin{align*}
	L(W) = \E_{(\mathbf{X}, \mathbf{y}) \sim P} \left[\frac{1}{2} (\langle \mathbf{X}, W \rangle - \mathbf{y})^2 \right]  + \lambda \| W \|_* := g(W) + h(W).
\end{align*}
\end{itemize}

Now we will show the role of moment matrix $\mathbf{V}$ in ``speeding up'' the convergence of proximal gradient descent in the following proposition. Here we focus on an easier case when $\mathbf{V}(w)$ does not depend on $w$\footnote{As is in the linear regression fixed design case, where $\mathbf{V}(w) =  \left( \mathbb{E}[ \xi(\mathbf{x})^2 \mathbf{x} \mathbf{x}^T] \right)^{1/2}$ does not depend on $w$.}.

Define the moment-adjusted proximal function and MadProx
\begin{align}
	\prox_{\eta, \mathbf{V}}(w) = \argmin_{u} \left[ \frac{1}{2\eta}\| u - w  \|_{\mathbf{V}}^2 + h(u) \right], \\
	\label{eq:mad-proximal}
	\text{MadProx:}\quad w_{t+1} = \prox_{\eta, \mathbf{V}}(w_t - \eta \mathbf{V}^{-1} \nabla g(w_t)).
\end{align}

\begin{proposition}[Moment-adjusted proximal]
	\label{thm:prox}
	Consider $L(w) = g(w) + h(w)$ as in \eqref{eq:non-smooth-decom}.
	Denote $\mathbf{H}$ as the Hessian of $g$, and define
	\begin{align*}
		\alpha \triangleq \min_{v}~ \lambda_{\min} \left( \mathbf{V}^{-1/2} \mathbf{H}(v) \mathbf{V}^{-1/2}  \right) >0 , \quad \gamma \triangleq \max_{v}~ \lambda_{\max} \left( \mathbf{V}^{-1/2} \mathbf{H}(v) \mathbf{V}^{-1/2}  \right) >0.
	\end{align*}
	Consider the MadProx updates defined in \eqref{eq:mad-proximal} with step-size $\eta = 1/\gamma$ and adjusting matrix $\mathbf{V}$. If
	$$
	T \geq \frac{\gamma}{\alpha} \log \left( \frac{\alpha}{2\epsilon} \|w_0 - w_* \|_{\mathbf{V}}^2 + 1 \right),
	$$
	we have $L(w_T) - \min_w L(w) \leq \epsilon.$
\end{proposition}
\begin{remark}
	One can see that MadProx
	implements moment-adjusted gradient (using implicit updates)
	because $w_{t+1}$ satisfies the implicit equation
	\begin{align*}
		w_{t+1} = w_{t} - \eta \mathbf{V}^{-1} (\nabla g(w_t) + \partial h(w_{t+1})),
	\end{align*}
	in comparison to the sub-gradient step (explicit updates)
	\begin{align*}
		w_{t+1} = w_{t} - \eta \mathbf{V}^{-1} (\nabla g(w_t) + \partial h(w_{t})).
	\end{align*}
	Remark that as in the GLMs case, the moment-adjusted idea speed up the computation as the number of proximal steps scales with adjusted condition number
	$\kappa_{\rm MadProx} \approx \sqrt{\kappa_{\rm GD}}$.
	However, to be fair, it can be computationally hard to implement each proximal step for a non-diagonal $\mathbf{V}$. Motivated from the diagonalizing idea in AdaGrad \citep{duchi2011adaptive}, one can substitute $\mathbf{V}$ by ${\rm diag}(\mathbf{V})$ to save the per-iteration computation.

\end{remark}

\section{Non-Convex Inference}
\label{sec:non-convex}

In this section, we study the non-asymptotic inference and optimization for stationary points of a smooth non-convex population landscape $L(\theta)$, via our proposed MasGrad.

\subsection{Inference and optimization for stationary points}
First we state a theorem that quantifies how well our proposed MasGrad achieves both the inference and optimization goal.

\begin{thm}[MasGrad: non-convex]
	\label{thm:non-convex}
	Let $L(w): \mathbb{R}^p \rightarrow \mathbb{R}$ be a smooth function. Recall $\mathbf{b}(w) = \nabla L(w)$, and $\mathbf{H}(w)$ being the Hessian matrix of $L$. $\mathbf{V}(w) \in \mathbb{R}^{p \times p}$ is a positive definite matrix.
	Assume
	\begin{align*}
		\gamma &\triangleq \max_{v, w}~ \lambda_{\max} \left( \mathbf{V}(w)^{-1/2} \mathbf{H}(v) \mathbf{V}(w)^{-1/2}  \right) >0.
	\end{align*}

  Consider the MasGrad updates $\theta_t$ in \eqref{eq: MasGrad} with step-size $\eta = 1/\gamma$, and the corresponding discretized diffusion $\xi_t$,
	\begin{align*}
		\xi_{t+1} = \xi_t - \eta \mathbf{V}(\xi_t)^{-1} \mathbf{b}(\xi_t)+ \sqrt{2\beta^{-1} \eta} \mathbf{g}_t,\quad \text{where $\beta = \frac{2n}{\eta}$}.
	\end{align*}
  Then for any precision $\epsilon, \delta > 0$, one can choose
	\begin{align}
	  T = \frac{2\gamma(L(\theta_0) - \min_\theta L(\theta)) + p\delta^2 }{\epsilon^2} \cdot (\max_\theta \|\mathbf{V}(\theta)\| \vee 1),~ \text{and}~~ n = \frac{T}{\delta^2},
	\end{align}
  such that
	\begin{align*}
		&(1)\quad D_{\rm TV}\left( \mu(\theta_t, t \in [T]),  \mu( \xi_t, t \in [T] ) \right) \leq  O_{\delta}(\delta), \\
		&(2)\quad \E \min_{t \leq T} \| \nabla L(\theta_t) \|\leq \epsilon, ~\E \min_{t \leq T} \| \nabla L(\xi_t) \|  \leq \epsilon,
	\end{align*}
	with in total $O_{\epsilon,\delta}(\epsilon^{-4} \delta^{-2})$ independent data samples.
\end{thm}

\begin{remark}
  \rm
	We would like to contrast the optimization part of the above theorem with the sample complexity result of classic SGD. To obtain an $\epsilon$-stationary point $w$ such that in expectation $\| \nabla L(w) \| \leq \epsilon$, SGD needs $O_\epsilon(\epsilon^{-4})$ iterations for non-convex smooth functions (with step size $\eta_t = \min\{ 1/\gamma, 1/\sqrt{t}\}$). Here we show that one can achieve this accuracy with the same dependence on $\epsilon$ with MasGrad, while being able to make statistical inference at the same time. And the additional price we pay for $\delta$-closeness in distribution for statistical inference is a factor of $\delta^{-2}$.

	The result can also be compared to Thm.~\ref{thm:converge} (the strongly convex case). In both cases, statistically, we have shown that the discretized diffusion $\xi_t$ tracks the non-asymptotic distribution of MasGrad $\theta_t$, as long as the data-generating process satisfies conditions like weak moment and bounded entropic distance to Gaussian. The distribution of $\xi_t$ is universal regardless of the specific data-generating distribution, and only depends on the moments $\V(\theta)^{-1}\b(\theta)$. In terms of optimization, to obtain an $\epsilon$-minimizer, the discretized diffusion approximation to MasGrad --- with the proper step-size $\eta$, and inverse temperature $\beta = 2n/\eta$ --- achieves the acceleration in the strongly convex case, and enjoys the same dependence on $\epsilon$ as SGD in the non-convex case in terms of sample complexity.
\end{remark}

\subsection{Why local inference}

For a general non-convex landscape, let us discuss why we focus on inference about local optima, or more precisely stationary points. Our Thm.~\ref{thm:non-convex} can be read as, within reasonable number of steps, the MasGrad converges to a population stationary point, and the distribution is well-described by the discretized Langevin diffusion. One can argue that the random perturbation introduced by the isotropic Gaussian noise in Langevin diffusion makes the process hard to converge to a typical saddle point. Therefore, intuitively, the MasGrad will converge to a distribution that is well concentrated near a certain local optima (depending on the initialization) as the temperature parameter $\beta^{-1} = \eta/2n$ is small. In this asymptotic low temperature regime, the Eyring-Kramer Law states that the transiting time from one local optimum to another local optimum, or the exiting time from a certain local optimum, is very long --- roughly $e^{\beta h}$ where $h$ is the depth of the basin of the local optimum \citep{bovier2004metastability, tzen2018local}. Therefore, a reasonable and tangible goal is to establish statistical inference for population local optima, for a particular initialization.

\section{Estimation and Computation of MasGrad Direction}
\label{sec:est-comp-mat-root}

We address in this section how to estimate and efficiently approximate the MasGrad direction $\V(\theta)^{-1} \b(\theta)$ at a current parameter location $\theta$. The estimation part undertakes a plug-in approach relying on the theory of self-normalized processes \citep{pena2008self}. For efficient computation of the pre-conditioning matrix, we devise a fast iterative algorithm to directly approximate the root of the inverse covariance matrix, which in a way resembles the advantage of quasi-Newton methods \citep{wright1999numerical}, however, with noticeable differences. The quasi-Newton methods approximate Hessian with first-order information, while MasGrad uses stochastic gradient information to approximate the root of the inverse covariance matrix as pre-conditioning. 
In this section we deliberately state all propositions working with general sample covariance matrix $\widehat{\Sigma}$ with dimension $d$, to emphasize that the results extend beyond the discussions for MasGrad.     

\subsection{Statistical estimation and self-normalized processes}
\label{sec:est-mat-root}

Recall that $\mathbf{V}(\theta)$ is the matrix root of the covariance. We estimate the moment-adjusted gradient direction $\mathbf{V}(\theta)^{-1} \mathbf{b}(\theta)$ at current location $\theta$, base on a mini-batch of size $n$. This section concerns this estimation part, borrowing tools from self-normalized processes. Define the sample estimates based on i.i.d data $z_i$ as
\begin{align*}
	\widehat{\b}(\theta) &\triangleq \frac{1}{n} \sum_{i=1}^n \nabla_{\theta} \ell(\theta, z_i) \\
	\widehat{\Sig}(\theta) &\triangleq \frac{1}{n-1} \sum_{i=1}^n [\nabla_{\theta}\ell(\theta, z_i) - \widehat{\b}(\theta)] \otimes [\nabla_{\theta}\ell(\theta, z_i) - \widehat{\b}(\theta)]
\end{align*}
and $\widehat{\V}(\theta)$ satisfies 
	$\widehat{\V}(\theta) \widehat{\V}(\theta)^T = \widehat{\Sig}(\theta),$
we will show that the plug-in approach $\widehat{\V}(\theta)^{-1} \widehat{\b}(\theta)$ estimates the population moment-adjusted gradient direction $\mathbf{V}(\theta)^{-1} \mathbf{b}(\theta)$ consistently at a parametric rate, in the fixed dimension setting.

\begin{proposition}[Connection to self-normalized processes]\label{lem:self-norm-process}
	Consider $\{x_i \in \mathbb{R}^d, 1\leq i \leq n\}$ i.i.d with mean $\mu$, $\bar{x}$ and $\widehat{\Sigma}$ to be sample mean vector and sample covariance. Consider $d \ll n$ and $\widehat{\Sigma}$ is invertible. Denote the centered moments
	\begin{align*}
		S_n \triangleq \sum_{i=1}^n (x_i - \mu), \quad  V_n^2 \triangleq \sum_{i=1}^n (x_i - \mu) \otimes (x_i - \mu)		
	\end{align*}
	and the multivariate self-normalized process
	\begin{align*}
		M_n \triangleq V_n^{-1} S_n \in \mathbb{R}^d.
	\end{align*}
	Then there exists $\widehat{V}$, which satisfies
	$
	\widehat{V} \widehat{V}^T = \widehat{\Sigma}
	$
	such that
	\begin{align*}
		\sqrt{n} \widehat{V}^{-1} (\bar{x} - \mu) =  M_n \cdot \sqrt{\frac{n-1}{n - \| M_n \|^2}}.
	\end{align*}
\end{proposition}

\begin{remark}
	\rm
	In the case of $d=1$, the above proposition reduces to a standard result in \cite{pena2008self}. In our matrix version,
		the proof relies on Sherman-Morrison-Woodbury matrix identity, together with a rank-one update formula for matrix root we derived in Lemma~\ref{lem:matrix-root} in Appendix~\ref{sec:proof}. Recall the Law of the Iterated Logarithm (LIL) on the norm of self-normalized process $\|M_n\|^2 \sim \log \log n$ (Theorem 14.11 in  \citep{pena2008self}, in the case when dimension is fixed), a direct application of the above formula implies
	\begin{align*}
		\widehat{\V}(\theta)^{-1} \left( \widehat{\b}(\theta) - \mathbf{b}(\theta) \right) = \frac{1}{\sqrt{n}} M_n \cdot \sqrt{\frac{n-1}{n - \| M_n \|^2}}  =  \frac{1 + O_{\bf p}(\log\log n/n)}{\sqrt{n}} M_n,
	\end{align*}
	where $M_n$ is a self-normalized process with asymptotic distribution being $\mathcal{N}(0, I_p)$. By Lemma~\ref{lem:consistency} in Appendix~\ref{sec:proof}, when $p\ll n$, the following approximation holds
	\begin{align*}
		\widehat{\V}(\theta)^{-1} \widehat{\b}(\theta) - \mathbf{V}(\theta)^{-1} \mathbf{b}(\theta) = \overbrace{\widehat{\V}(\theta)^{-1} \left( \widehat{\b}(\theta) - \mathbf{b}(\theta) \right)}^{\text{self-normalized processes}} + O_{\bf p} \left(\sqrt{\frac{p \log n}{n}} \right), 
	\end{align*}
	where the approximation is with respect to $\ell_2$ norm.
	All together, the above implies that one can estimate $\mathbf{V}(\theta)^{-1} \mathbf{b}(\theta)$ consistently in the fixed dimension $p$ and large $n$ setting. 
\end{remark}

\subsection{Efficient computation via direct rank-one updates}
\label{sec:comp-mat-root}

In this section we devise a fast iterative formula for calculating $\widehat{\V}(\theta)^{-1}$ directly via rank-one updates.

Recall the brute-force approach of calculating $\widehat{\Sig}(\theta)$ first
then solving for the inverse root $\widehat{\V}(\theta)^{-1}$ involves $O(n p^2 + p^3)$ complexity in the computation. Instead, we will provide an algorithm that approximates $\widehat{\V}(\theta)^{-1}$ directly through iterative rank-one updates, that is only $O(np^2)$ in complexity, utilizing the fact that the sample covariance is a finite sum of rank-one matrices. To the best of our knowledge, this direct approach of calculating root of inverse covariance matrix is new. 

\begin{proposition}[Iterative rank-one updates of matrix inverse root]
	\label{lem:rank-one-mat-root-inv}
	Initialize $H_0 = I_d$, and define the recursive rank-one updates for matrix inverse root, for $v_i \in \mathbb{R}^d$
	\begin{align}
		\label{eq:rank-one-update}
		H_{i+1} = H_{i} - \frac{1}{\alpha_i} H_i v_{i+1} v_{i+1}^T H_i^T H_i
	\end{align}
	with $\alpha_i \triangleq (1 + \sqrt{1+v_{i+1}^T H_i^T H_i v_{i+1}})\sqrt{1+v_{i+1}^T H_i^T H_i v_{i+1}} \in \mathbb{R}$.
	Then for all $n$, $H_n$ is the matrix inverse root of $I_d + \sum_{i=1}^{n} v_i \otimes v_i$.
	In other words, define $V_n \triangleq H_n^{-1},$
	then $V_n V_n^T =  I_d + \sum_{i=1}^{n} v_i \otimes v_i.$
\end{proposition}
\begin{remark}
	\rm
	One can directly apply the above result to evaluate $\widehat{\V}(\theta)^{-1}$ efficiently. 
	Define $v_i = \nabla_{\theta}\ell(\theta, \z_i) - \widehat{\b}(\theta)$, one can use \eqref{eq:rank-one-update} in the above proposition for fast iterative calculations, and that
	$$
	\left( \sqrt{n-1} H_{n} \right)^{-1} \left( \sqrt{n-1} H_{n}^T \right)^{-1} = \frac{1}{n-1} I_d + \widehat{\Sig} \approx \widehat{\Sig}.
	$$
	Therefore $\widehat{\V}(\theta)^{-1}$ is approximated by $\sqrt{n-1} H_{n}$. We remark that the quality of the approximation depends on the spectral decay of the true covariance $\Sigma$.

	For each iteration, the computation complexity for \eqref{eq:rank-one-update} is $4d^2$, with some careful design in calculation: it takes $d^2$ operations to calculate $H_i v_{i+1} \in \mathbb{R}^d$, then an additional $d^2$ to calculate $(H_i v_{i+1})^T H_i \in \mathbb{R}^d$, another $d^2$ operations for multiplication of rank-one vectors $H_i v_{i+1} \times (H_i v_{i+1})^T H_i$, and finally $d^2$ operations for matrix addition. Hence, the total complexity is $O(nd^2)$ (for MasGrad, simply substitute $d = p$).	
\end{remark}

\subsection{Optimal updates for online least-squares}
\label{sec:optim-updat-online}

In the case of a least-squares loss $\ell(\theta, z) = \frac12 (y - x^T \theta)^2$, we offer a simple and efficient online rule for estimating $\mathbf{V}(\theta)$ without any accuracy loss compared with offline counterparts. This is based on the fact that the data points $z_i$ and the parameter $\theta$ can be ``decoupled'' in least-squares. To show this, first write the covariance as
\[
\begin{aligned}
\mathbf{V}(\theta)^2 &= \Cov[(\mathbf{y} - \mathbf{x}^T \theta) \mathbf{x}]\\
&= \Cov(\mathbf{x}\mathbf{x}^T \theta) + \Cov(\mathbf{y}\mathbf{x}) - 2 \Cov(\mathbf{x}\mathbf{x}^T \theta, \mathbf{x}\mathbf{y}).
\end{aligned}
\]
To efficiently estimate $\Cov(\mathbf{x}\mathbf{x}^T \theta)$ in an online fashion, we observe that
\begin{equation}\label{eq:cov_two}
\Cov(\mathbf{x}\mathbf{x}^T \theta) = \E_{\mathbf{z} \sim P} (\mathbf{x}\mathbf{x}^T \theta \theta^T \mathbf{x}\mathbf{x}^T) - \left[\E_{\mathbf{z} \sim P} (\mathbf{x}\mathbf{x}^T \theta) \right] \left[\E_{\mathbf{z} \sim P} (\mathbf{x}\mathbf{x}^T \theta) \right]^T.
\end{equation}
Recalling $\otimes_{\text{K}}$ denotes the Kronecker product and letting $\vecs(X)$ be the vector that is formed by stacking the columns of $X$ into a single column, we express $\mathbf{x}\mathbf{x}^T \theta \theta^T \mathbf{x}\mathbf{x}^T$ as
\[
\vecs (\mathbf{x}\mathbf{x}^T \theta \theta^T \mathbf{x}\mathbf{x}^T) = \left[(\mathbf{x} \mathbf{x}^T) \otimes_{\text{K}} (\mathbf{x}\mathbf{x}^T) \right] (\theta \otimes_{\text{K}} \theta).
\]
This expression shows that
\begin{equation}\nonumber
\vecs\left[\E_{\mathbf{z} \sim P} (\mathbf{x}\mathbf{x}^T \theta \theta^T \mathbf{x}\mathbf{x}^T) \right] = \E_{\mathbf{z} \sim P}\left[(\mathbf{x} \mathbf{x}^T) \otimes_{\text{K}} (\mathbf{x}\mathbf{x}^T) \right] (\theta \otimes_{\text{K}} \theta).
\end{equation}
Accordingly, one can simply keep track of $\sum_{i=1}^t (\mathbf{x}_i \mathbf{x}_i^T) \otimes_{\text{K}} (\mathbf{x}_i\mathbf{x}_i^T)$ in the online setting and estimate $\E_{\mathbf{z} \sim P} (\mathbf{x}\mathbf{x}^T \theta \theta^T \mathbf{x}\mathbf{x}^T)$ through mapping the vector
\[
\left[ \frac{\sum_{i=1}^t (\mathbf{x}_i \mathbf{x}_i^T) \otimes_{\text{K}} (\mathbf{x}_i\mathbf{x}_i^T)}{t} \right] (\theta \otimes_{\text{K}} \theta).
\]
to its associated matrix. It remains to estimate $\E_{\mathbf{z} \sim P} (\mathbf{x}\mathbf{x}^T \theta)$ in \eqref{eq:cov_two}. Recognizing $\E_{\mathbf{z} \sim P} (\mathbf{x}\mathbf{x}^T \theta) = \left[ \E_{\mathbf{z} \sim P} (\mathbf{x}\mathbf{x}^T) \right] \theta$, this can be done by simply recording the sum $\mathbf{x}_1\mathbf{x}_1^T + \cdots + \mathbf{x}_t\mathbf{x}_t^T$ in an online manner and replacing $\E_{\mathbf{z} \sim P} (\mathbf{x}\mathbf{x}^T)$ by the average $(\mathbf{x}_1\mathbf{x}_1^T + \cdots + \mathbf{x}_t\mathbf{x}_t^T)/t$. Likewise, $\Cov(\mathbf{y}\mathbf{x})$ and $\Cov(\mathbf{x}\mathbf{x}^T \theta, \mathbf{x}\mathbf{y})$ can be estimated in the online setting regardless of a varying $\theta$. We omit this part in the interest of space.

\section{Numerical Experiments}

In this section we present results for numerical experiments. Full details of the experiments are deferred to Appendix~\ref{sec:exp-details}.

\smallskip

\noindent \textbf{Linear models.} \quad 

The first numerical example is the simple linear regression, as in Fig.~\ref{fig:linear}. Here we generate two plots as a proof of concept. The top one summarizes the trajectory of several methods for inference --- our proposed \textit{MasGrad}, the discretized diffusion approximation \textit{diff\_MasGrad}, as well as the classical \textit{SGD}, and the diffusion approximation \textit{diff\_SGD} --- with the confidence intervals (95\% coverage) at each time step $t$. In this convex setting, we can solve for the global optimum, which is labeled as the \textit{truth}. Here the mini-batch size is $n=50$. We run $100$ independent chains to calculate the confidence intervals at each step. We look at the low dimensional case $p = 4$, and the four subfigures (on top) each corresponds to one
coordinate of the parameter $w_i, i \in [p]$. The $x$-axis is $t$, the time of the evolution, and $y$-axis is the value of the parameter $w$. We remark that \textit{MasGrad} and \textit{diff\_MasGrad} are path-wise close in terms of distribution, which verifies our statistical theory in Thm.~\ref{thm:couple.p}. This also holds for \textit{GD} and \textit{diff\_GD}.
Remark that in this simulation, the condition number of the empirical Gram matrix is $30.98$, and the first and third coordinates have very small population eigenvalues, which explains why in those coordinates \textit{MasGrad} has significant acceleration compared to \textit{SGD} as shown in the figure. To be fair, at each time step, both MasGrad and SGD sample the same amount of data, and the step-size is chosen as in Thm.~\ref{thm:acceleration}. All four chains start with the same random initialization.

To examine the optimization side of the story, we plot the logarithm of the $\ell_2$-error according to time $t$, for \textit{diff\_MasGrad} and \textit{diff\_SGD}, in the bottom plot. Remark that the error bar quantifies the confidence interval for the log error. In theory, we should expect that the slope of MasGrad is twice that of the slope of SGD. In simulation, it seems that the acceleration is slightly better than what the theory predicts.
We would like to remark that compared to GD, in which different coordinates make uneven progress (fast progress in the second and fourth coordinates, but slow on the others), MasGrad adaptively adjusts the relative step-size on each coordinate for synchronized progress. This effect has also been observed in AdaGrad and natural gradient descent.

\begin{figure}[pht]
  \centering
\includegraphics[width = 0.9\textwidth]{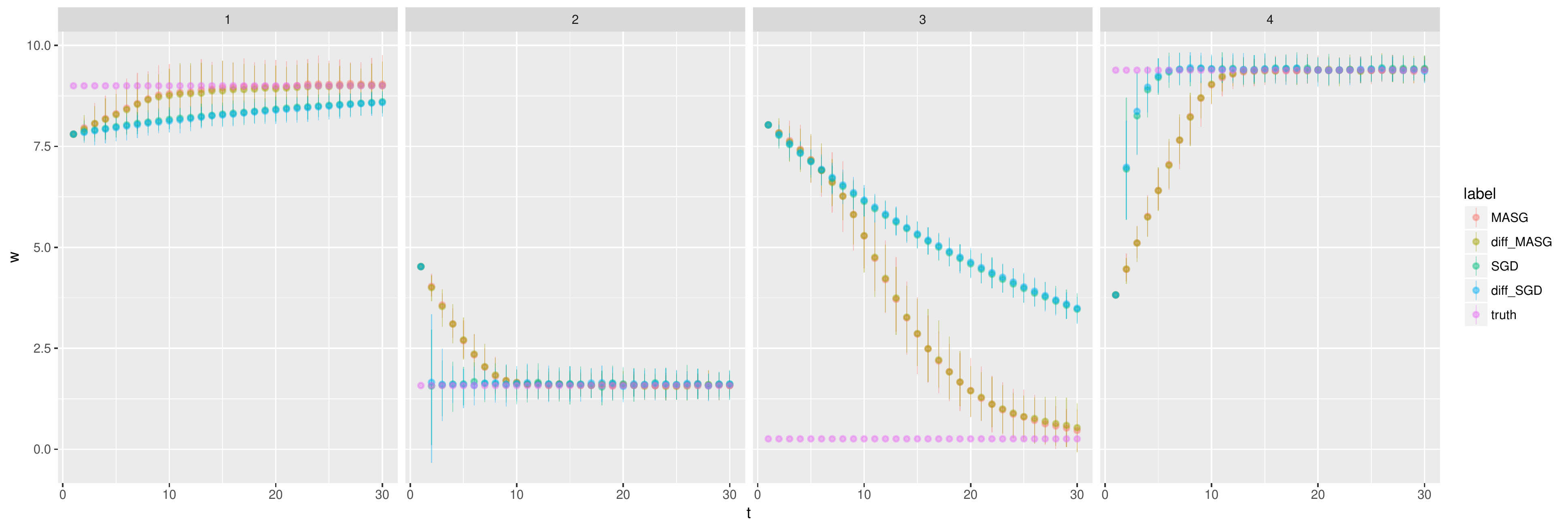}
\includegraphics[width = 0.9\textwidth]{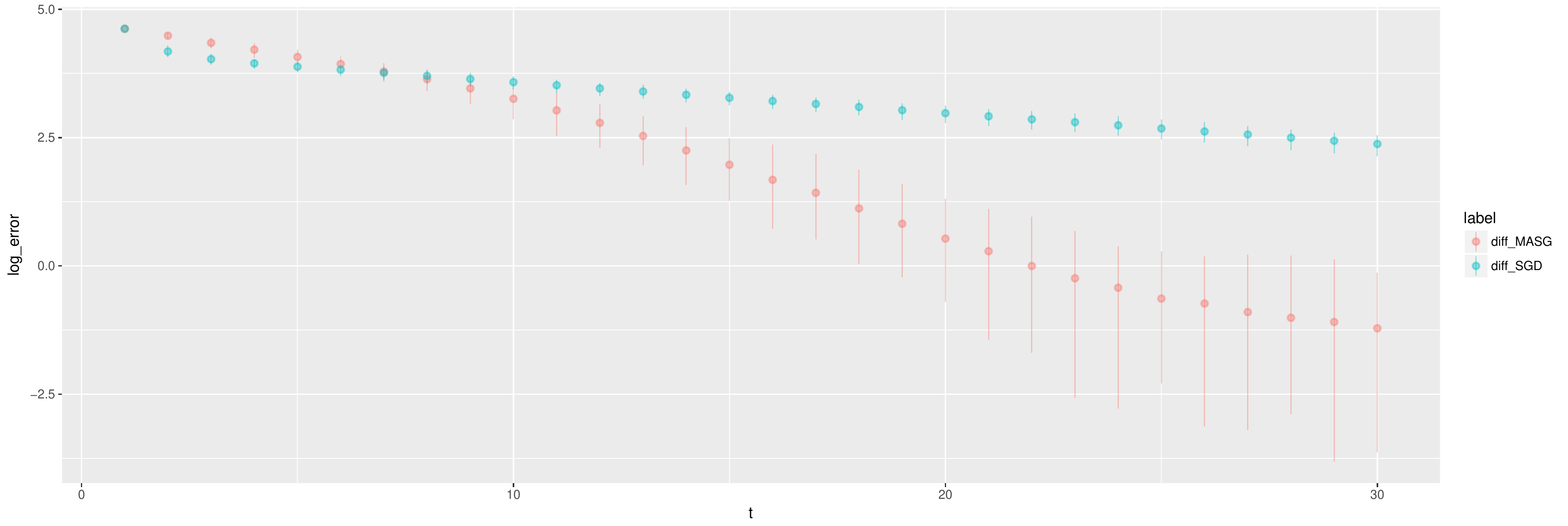}
\caption{Linear regression}
\label{fig:linear}
\end{figure}

\noindent \textbf{Logistic model.} \quad
Fig.~\ref{fig:logistic} illustrates the acceleration for inference in logistics regression. The figure should be read the same way as in the linear case. In this case, we sample a much larger number of samples ($N = 500$) and then use the GLMs package in R to fit the global optimum. For MasGrad and SGD, we generate bootstrap subsamples ($n = 25$) to evaluate stochastic descents at each iteration. Again, we run $100$ independent chains to calculate the confidence interval at each step. 
In this case, there is no theoretically optimal way of choosing the step-size, so we choose the same step-size ($\eta = 0.2$) for both MasGrad and SGD.

Statistically, the \textit{MasGrad} and \textit{diff\_MasGrad} are close in distribution when $t < 100$, and they both reach a stationary distribution after around $50$ steps, simultaneously for all $p=4$ coordinates. Then the distribution fluctuates around stationarity. However, \textit{GD} and \textit{diff\_GD} make much slower progress, and they fail to reach the global optimum in $100$ steps.
For optimization, empirically the acceleration in the log error plot seems to be better than what the theoretical results predict. Remark that the confidence intervals are on the scale of log error, therefore, it is negative-skewed.

\begin{figure}[pht]
  \centering
\includegraphics[width = 0.9\textwidth]{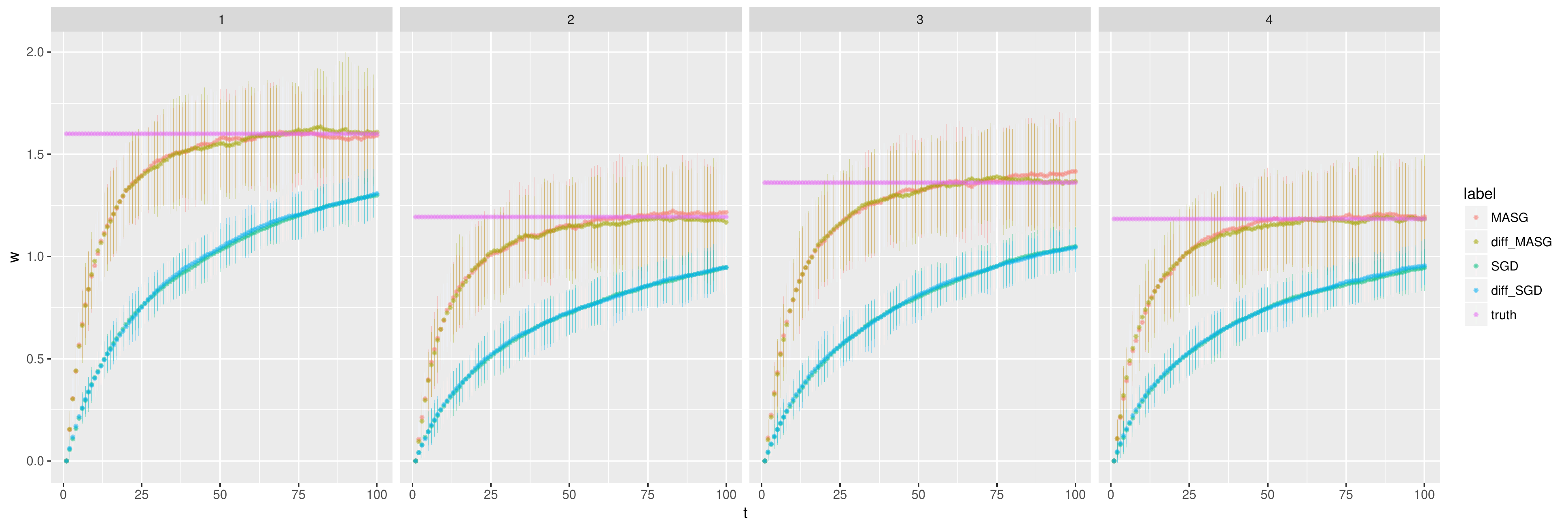}
\includegraphics[width = 0.9\textwidth]{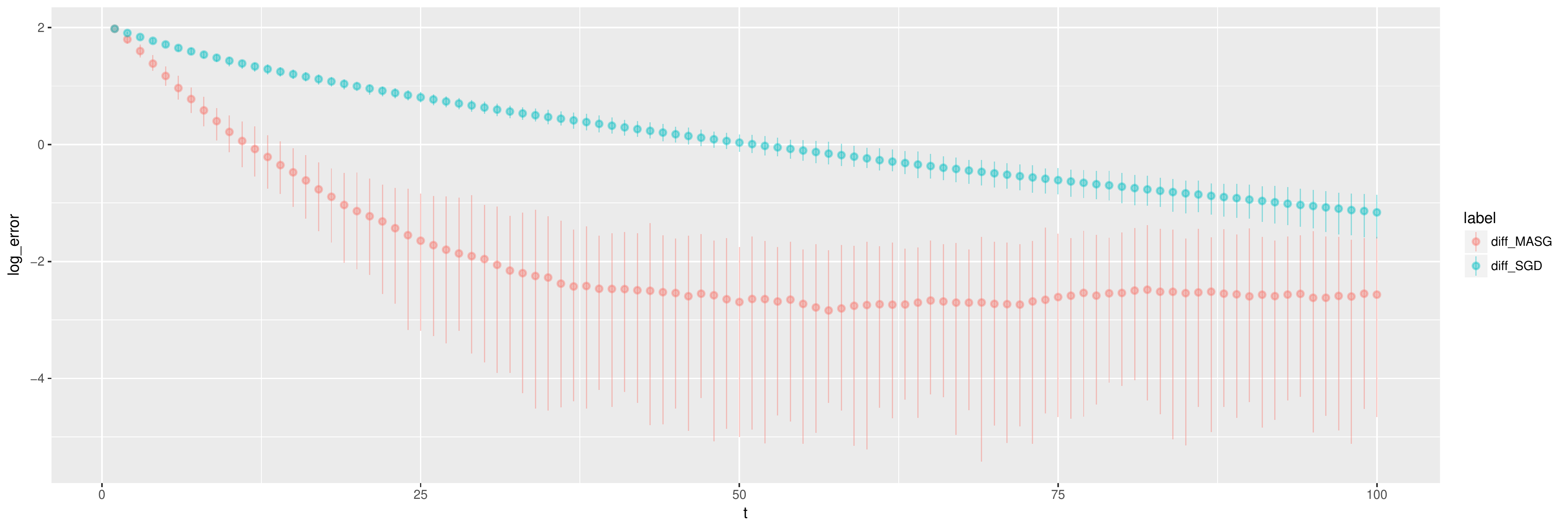}
\caption{Logistic regression}
\label{fig:logistic}
\end{figure}

\noindent \textbf{Gaussian mixture.} \quad
Here we showcase inference via MasGrad for non-convex case, using the Gaussian mixture model. We will consider a simple setting: the data $z_i \in \mathbb{R}^n, 1\leq i \leq [N]$ generated from a mixture of $p$ Gaussians, with mean $[\theta_1, \theta_2,\ldots, \theta_p]\triangleq \theta$ respectively, and variance $\sigma^2$. The goal is to infer the unknown mean vector $\theta \in \mathbb{R}^p$. The problem is non-convex due to the mixture nature: the maximum likelihood is multimodal, as we can shuffle the coordinates of $\theta$ to obtain the equivalent class of local optima.

Fig.~\ref{fig:mixture} illustrates the acceleration for inference in the Gaussian mixture model. Here we run two simulations, according to the difficulty (or separability) of the problem defined as the signal-to-noise ratio ${\rm SNR} \triangleq \min_{i\neq j} |\theta_{i} - \theta_{j}|/\sigma$. The top one is for the easy case with ${\rm SNR} = 3.3$ and the bottom one for the hard case with ${\rm SNR} = 1$. In both simulations, $\theta = (1,2,3) \in \mathbb{R}^3$, and we choose a random initial point to start the chains. The plot is presented as before. At each iteration, we subsample $n = 20$ data points to calculate the decent direction, and the step-size is fixed to be $\eta = 0.05$. Remark that there are many population local optima (at least $3! = 6$), and both \textit{MasGrad} and \textit{diff\_MasGrad} seem to be able to find a good local optimum relatively quickly (which concentrates near a permutation of $1,2,3$ for each coordinate), compared to \textit{SGD} and \textit{diff\_SGD}. The acceleration effect in both cases seems to be apparent. Again, we want to emphasize that the convergence for each coordinate in MasGrad seems to happen around the same number of iterations, which is not true for SGD.

\begin{figure}[pht]
\centering
\includegraphics[width =0.9\textwidth]{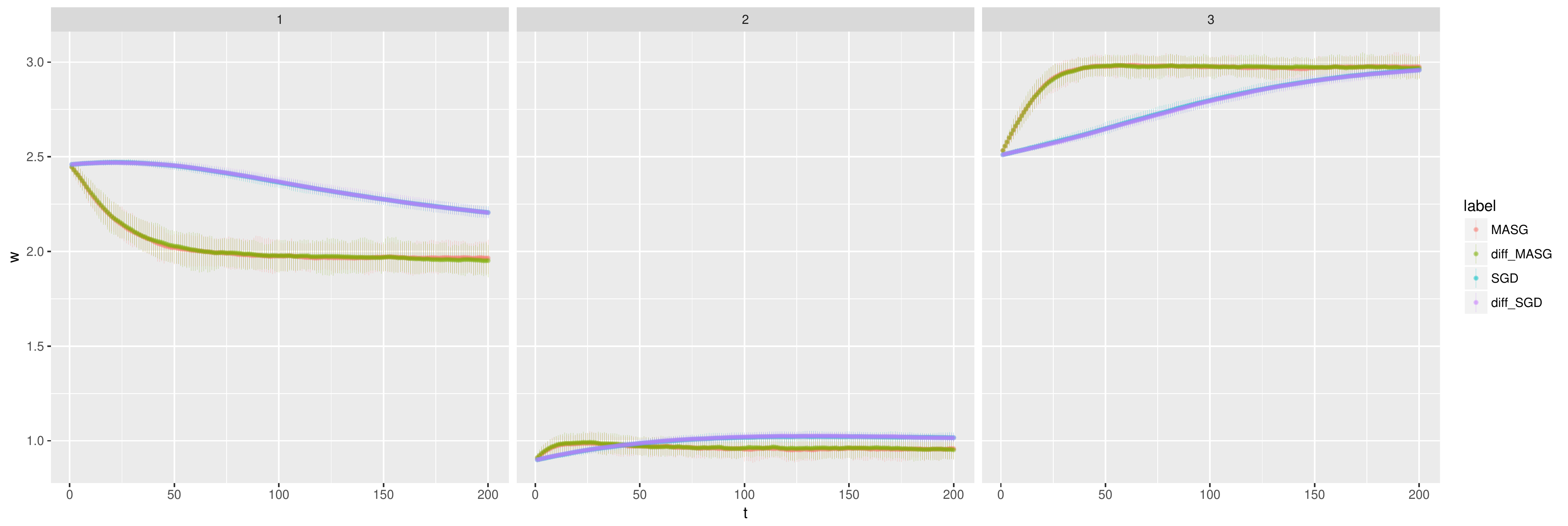}
\includegraphics[width =0.9\textwidth]{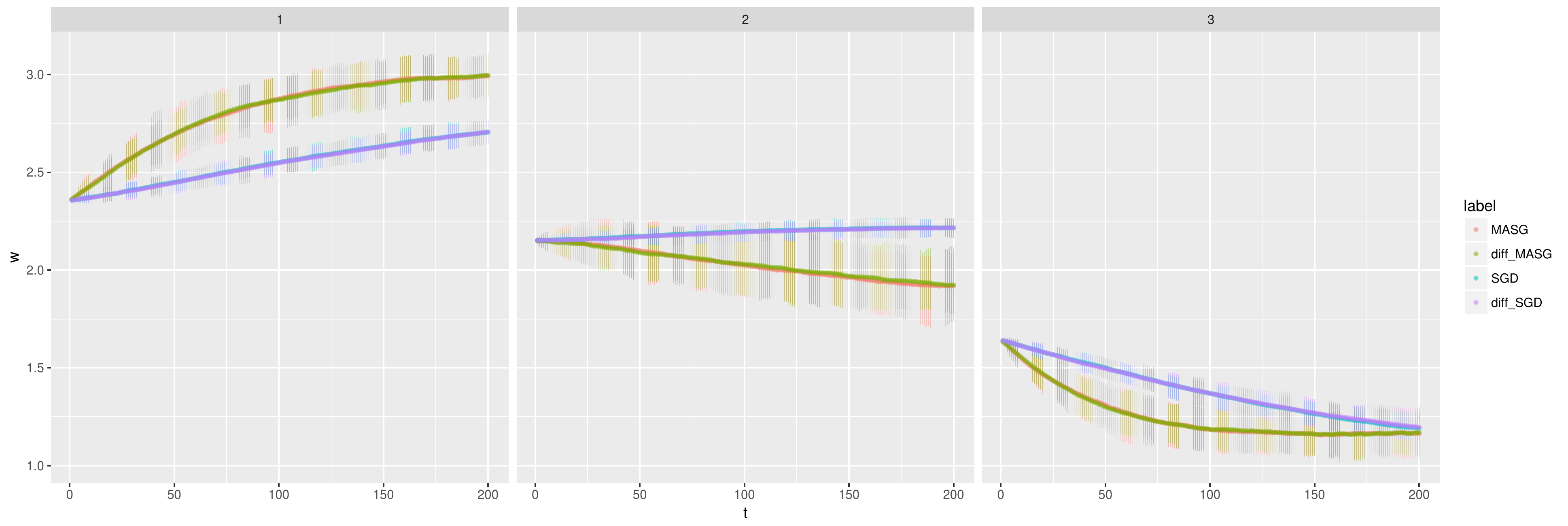}
\caption{Gaussian mixture}
\label{fig:mixture}
\end{figure}

\noindent \textbf{Shallow neural networks.} \quad
We also run MasGrad on a two-layer ReLU neural network, as a proof of concept for non-convex models. Define the ReLU activation $\sigma(x) = \max(x, 0)$, a two-layer neural network (with $k$ hidden units) represents a function
\begin{align*}
	f_w(x) = \sigma(W_2 \sigma(W_1 x)), \quad \text{where $x \in \mathbb{R}^d$, $w = \{ W_1 \in \mathbb{R}^{k \times d}, W_2 \in \mathbb{R}^{1 \times k} \}$}.
\end{align*}
In our experiment, we work with the square loss
$\ell(w, (x,y)) = \frac{1}{2}(y - f_w(x))^2.$
The gradients can be calculated through back-propagation. 
In this case, it is harder to calculate the global optimum; instead, in order to compare the \textit{diff\_MasGrad} and \textit{SGD}, we run $50$ experiments with random initializations to explore the population landscape.

For each experiment (as illustrated in the top figure in Fig.~\ref{fig:neural}), we randomly initialize the weights using standard Gaussians.  As usual, we run 100 independent chains with the same initial points for \textit{diff\_MasGrad} and \textit{SGD} to calculate the confidence interval. As anticipated, the distribution is rather non-Gaussian (for instance, in coordinate $2$ and $6$). We run the chain for 100 steps, and then evaluate the population loss function for the two methods. Out of the $50$ experiments, $45/50= 90\%$ of the time the population loss returned by \textit{diff\_MasGrad} is much smaller than that of the \textit{SGD}. The bottom figure in Fig.~\ref{fig:neural} plots the histogram (dotplot using ggplot2 \citep{Wickham:2009aa}) of the population error (test accuracy). Empirically, the $\textit{diff\_MasGrad}$ seems to converge to ``better'' local optima most of the time. There could be several explanations: first, MasGrad uses better local geometry (similar to natural gradient) so that it induces better implicit regularization; second, MasGrad as an optimization method accelerates the chain so that it mixes to a local optima faster, compared to SGD which may not yet converge within a certain time budget.

\begin{figure}[pht]
\includegraphics[width = 0.9\textwidth]{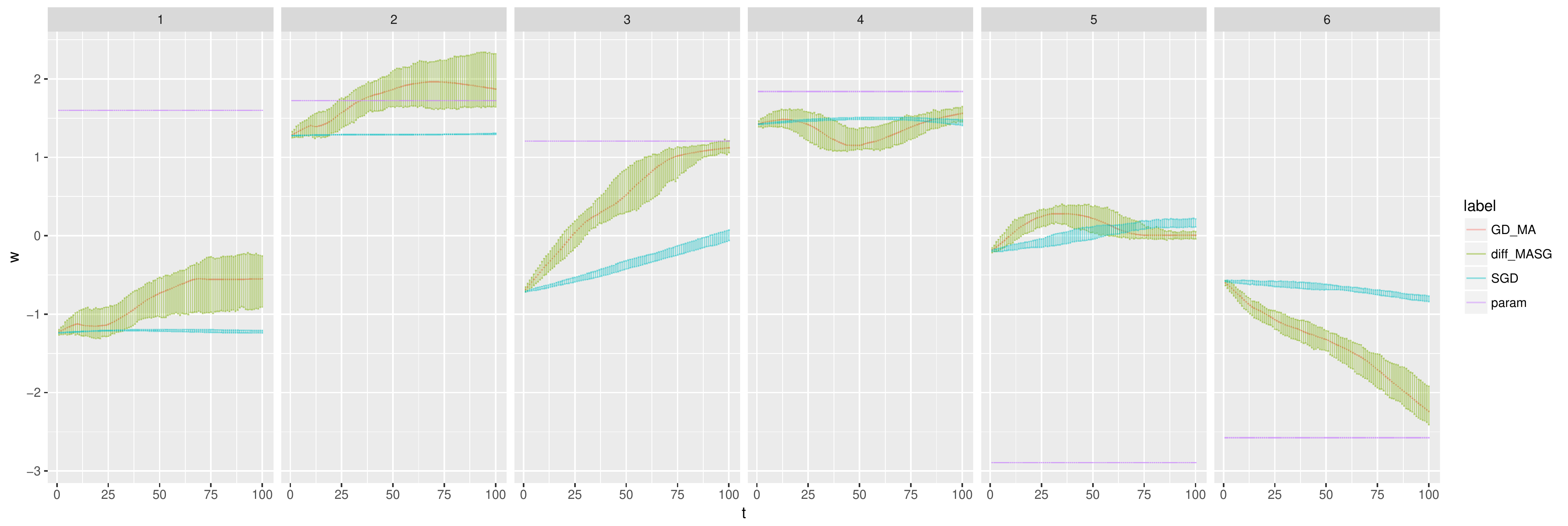}
\includegraphics[width = 0.9\textwidth]{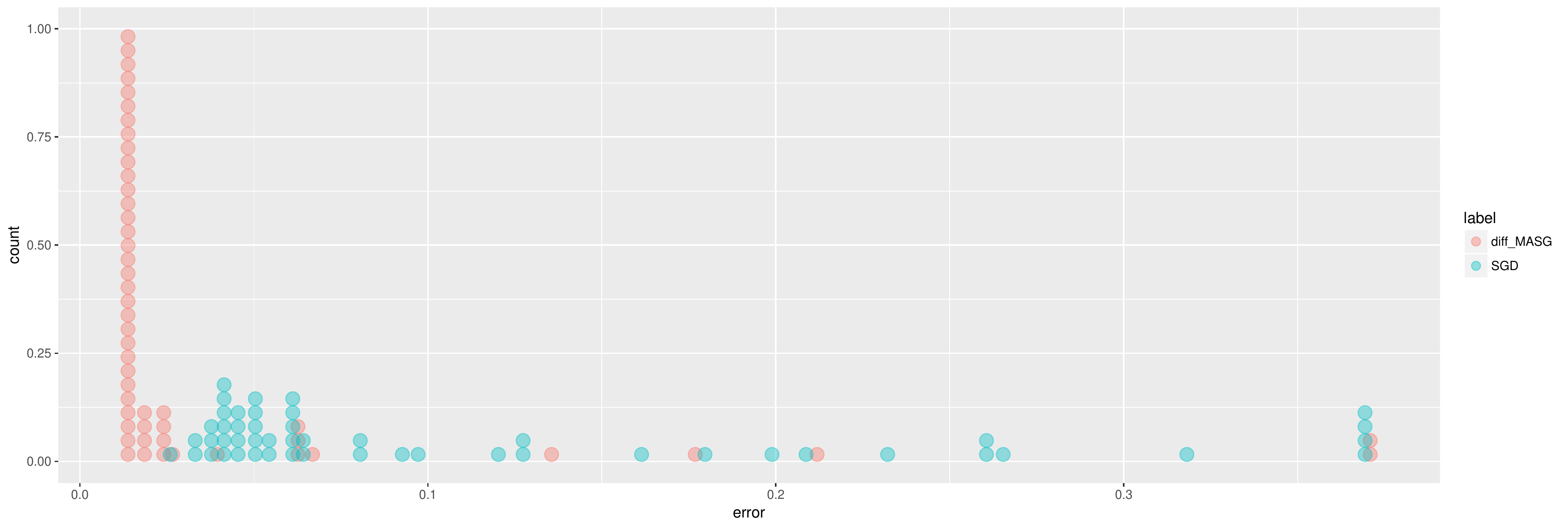}
\centering
\caption{Shallow neural nets}
\label{fig:neural}
\end{figure}

\section{Further Discussions}

Let us continue to discuss more about $\mathbf{V}(\theta_t)$. Note that in the fixed-dimension setting, one can estimate the covariance matrix of the gradient $\nabla \ell(\theta, \mathbf{z})$ using the empirical version with $N$ independent samples, when $N$ is large. Let us be more careful in this statement: (1) When the population landscape is convex, then the global optimum of $\widehat{L}_N(\theta)$ and $L(\theta)$ are within $1/\sqrt{N}$. We can always treat $\widehat{L}_N(\theta)$ as the population version and at each step we bootstrap subsamples of size $n$ to evaluate the stochastic gradients, adjusted using the empirical covariance $\widehat{\mathbf{V}}_N$ calculated using $N$ data points. Intuitively, when $\eta < O(n/N)$ (so that $\beta > N$), we know the MasGrad will concentrate near the optimum of $\widehat{L}_N(\theta)$ with better accuracy than $1/\sqrt{N}$. (2) In the non-convex case, things become unclear. However, under stronger conditions such as strongly Morse \citep{mei2016landscape}, i.e., when there is nice one-to-one correspondence between the stationary points of $\widehat{L}_N(\theta)$ and $L(\theta)$, one may still use the bootstrap idea above with $\widehat{\mathbf{V}}_N$. (3) Computation of $\widehat{\mathbf{V}}_N$ and its inverse could be burdensome, thus one may want to use the efficient rank-one updates designed in Section~\ref{sec:optim-updat-online}, or to calculate a diagonalized version of $\widehat{\mathbf{V}}_N$ as done in AdaGrad \citep{duchi2011adaptive}. (4) To have fully rigorous non-asymptotic theory in the case where $\mathbf{V}$ is known, one may require involved tools from self-normalized processes \citep{pena2008self} to establish a similar version of entropic CLT for multivariate self-normalized processes, where we standardize $\widehat{\E}_n[\nabla \ell(\theta, \mathbf{z})]$ by the empirical covariance matrix $\widehat{\mathbf{V}}_n$ calculated based on the same samples. To the best of our knowledge, this is an ambitious and challenging goal that is beyond the scope and focus of the current paper.

We would like to conclude this section by discussing the connections between pre-conditioning methods and our moment-adjusting method.
Pre-conditioning considers performing a linear transformation $\xi = A^{-1} \theta$ on the original parameter space on $\theta$. In other words, consider $\tilde{L}(\xi) \triangleq L(A\xi)$, and perform the updates on $\xi$ yields
\begin{align*}
	  \xi_{t+1} = \xi_{t} - \eta \nabla_\xi \tilde{L}(\xi)= \xi_{t} - \eta A \mathbf{b} (A \xi_t) ~\Rightarrow~ \theta_{t+1} =  \theta_{t} - \eta A^2 \mathbf{b}(\theta_t),
\end{align*}
Therefore, in the noiseless case, the moment-adjusting method is equivalent to pre-conditioning when the moment matrix $\mathbf{V}(\theta)$ is a constant matrix w.r.t. $\theta$. However, in Langevin diffusion when the isotropic Gaussian noise is presented, the connection becomes more subtle --- as $\mathbf{V}^{-1} (\theta) \mathbf{b}(\theta)$ may not be the gradient vector field for any function. The moment-adjusting idea motivated from standardizing noise in statistics is different from the pre-conditioning idea in optimization. We would also like to point out that a nice idea using Hessian information to speed up the Langevin diffusion for sampling from log-concave distribution has been considered in \cite{dalalyan2017theoretical}. Remark that we use the moment matrix at the current
point $\theta_t$ (time varying) instead of the optimal point $\theta_*$ (which is unknown). We also use the matrix root instead of the covariance matrix itself. In the case when the model is well-specified and the loss function chosen to be the negative log-likelihood, the $V(\theta_*)$ is the root of the Fisher information matrix.

\section*{Supplemental Materials}
Due to space constraints, we have
relegated further discussion of Langevin diffusion to Appendix~\ref{sec:continuous-langevin}, the detailed proofs to Appendix~\ref{sec:proof}, and remaining details about experiments to Appendix~\ref{sec:exp-details} in the Supplement to ``Statistical Inference for the Population Landscape via Moment-Adjusted Stochastic Gradients."

\section*{Acknowledgement}
The authors would like to thank the Associate Editor and the anonymous referees for the constructive feedback that significantly improves the content and presentation of the paper.

\bibliographystyle{abbrvnat}
\bibliography{ref}

\appendix

\section{Discussions on continuous time Langevin diffusion}
\label{sec:continuous-langevin}

In this section we will provide non-asymptotic bounds on the closeness of the discretized and the continuous time Langevin diffusion, both in terms of the Wasserstein-$2$ distance, and the entropic distance. Let us introduce few notations within this section. Denote $\mathbf{h}(x) = \mathbf{V}(x)^{-1} \mathbf{b}(x)$, and let's define two processes $\theta_{t}$ and $\xi_{t}$ with the same initial position $\xi_0$ as follows
\begin{align}
	\text{Continuous}: \quad \theta_t &=  \theta_0 - \int_0^t \mathbf{h}(\theta_s) ds + \sqrt{2\beta^{-1}} \int_0^t dB_s, \label{eq:conti}\\
	\text{Interpolation}: \quad \xi_{t} &= \theta_{0} - \int_0^t \mathbf{h}(\xi_{\lfloor s/\eta \rfloor \eta}) ds + \sqrt{2\beta^{-1}} \int_0^t dB_s. \label{eq:discr}
\end{align}
Here $\theta_{t}$ is the continuous time Langevin diffusion, and $\xi_{t}$ is an interpolation of the discretization process $\xi_{k}$ in Eqn.~\eqref{eq:discrete.diff}: for any integer $k$, the marginal distribution of $\xi_{k\eta}$ is the same as $\xi_{k}$, and is well-defined for any $t \in [(k-1)\eta, k\eta]$. Under the following standard assumptions, we can establish Lemma~\ref{lem:wass} for Wasserstein distance and Lemma~\ref{lem:discrete-to-continuous} for relative entropy.
\begin{assumption}[Lipschitz]
	\label{asmp:Lipschitz}
	Assume that $\mathbf{h}(\cdot)$ is $\ell$-Lipschitz,
	\begin{align*}
		\| \mathbf{h}(x) - \mathbf{h}(y) \| \leq \ell \| x - y \|, \quad \forall x, y \in \mathbb{R}^d.
	\end{align*}
\end{assumption}

\begin{assumption}[Boundedness]
	\label{asmp:bound}
	Assume that $\mathbf{h}(\cdot)$ is $M$-bounded,
	\begin{align*}
		\| \mathbf{h}(x) \| \leq M, \quad \forall x \in \mathbb{R}^d.
	\end{align*}
\end{assumption}

\begin{assumption}[Expansiveness]
	\label{asmp:expansive}
	Assume that $x \mapsto x  - \eta \mathbf{h}(x)$ is $\delta$-expansive,
	\begin{align*}
		\| \left(x - \eta \mathbf{h}(x) \right) - \left(y -\eta \mathbf{h}(y) \right) \| \leq \delta \| x - y \|, \quad \forall x, y \in \mathbb{R}^d.
	\end{align*}
\end{assumption}

\begin{lem}[Wasserstein]
	\label{lem:wass}
	Let $W_2(\mu, \nu)$ denote the Wasserstein-$2$ distance,
	\begin{align*}
		W_2(\mu, \nu) \triangleq \left\{ \inf_{\gamma \in \Gamma(\mu, \nu)} \int \| x - y \|^2 d \gamma(x, y) \right\}^{1/2}, \quad \text{$\Gamma(\mu, \nu)$ are all couplings of $\mu, \nu$.}
	\end{align*}
	Under Assumptions~\ref{asmp:Lipschitz}, \ref{asmp:bound} and \ref{asmp:expansive}, the Wasserstein-$2$ distance between the $\xi_{k\eta}$ in \eqref{eq:discr} and $\theta_{k\eta}$ in \eqref{eq:conti} satisfies
	\begin{align*}
		W_2\left(\mu(\xi_{k\eta}), \mu(\theta_{k\eta}) \right) \leq \left(\frac{2\ell^2 M^2 }{3} \eta^4 + 2\ell^2 p \cdot \beta^{-1} \eta^3 \right)^{1/2} \cdot \sum_{i=0}^{k-1} \delta^i.
	\end{align*}
\end{lem}

\begin{remark}
	\rm
	Let's make few remarks to dissect the non-asymptotic upper bound in Lemma~\ref{lem:wass}. Here we borrow Lemma 3.7 in \cite{hardt2015train}, which accounts for the expansiveness of updates induced by the vector field $\mathbf{h}$, as follows:
	\begin{enumerate}
		\item If $\mathbf{h}$ is $\ell$-smooth, then $x - \eta \mathbf{h}(x)$ is $(1+\eta \ell)$-expansive;
		\item If in addition $\mathbf{h} = \nabla U$, where $U$ is a convex, then for $\eta \leq \frac{2}{\ell}$, $x - \eta \mathbf{h}(x)$ is $1$-expansive;
		\item If in addition $U$ is $\alpha$-strongly convex, then for $\eta \leq \frac{2}{\alpha + \ell}$, $x - \eta \mathbf{h}(x)$ is $(1 - \frac{\eta \alpha \ell}{\alpha + \ell} )$-expansive.
	\end{enumerate}
	First, let's focus on the dependence of $\eta$ and $k$ in the Wasserstein bound. Plug in $\beta = \frac{2n}{\eta}$, we discuss the three cases for the expansiveness parameter $\delta$.
	\begin{enumerate}
		\item Smooth non-convex: $\delta = 1 + \eta \ell$, we have $W_2\left(\mu(\xi_{k\eta}), \mu(\theta_{k\eta}) \right) \leq O(\eta e^{\ell k\eta})$.
		\item Convex: $\delta = 1$, the Wasserstein-$2$ distance reads $W_2\left(\mu(\xi_{k\eta}), \mu(\theta_{k\eta}) \right) \leq O(k\eta^2)$.
		\item Strongly convex: $\delta = 1 - \frac{\eta \alpha \ell}{\alpha + \ell}$, we have $W_2\left(\mu(\xi_{k\eta}), \mu(\theta_{k\eta}) \right) \leq O(\eta \frac{\alpha + \ell}{\alpha\ell})$.
	\end{enumerate}
	In the convex and strongly convex cases, the Wasserstein bound depends on $k$ in a desirable weak manner, as one utilizes the non-expansiveness of the vector fields $\mathbf{h}$.
	In the most general smooth non-convex case, the Wasserstein bound $O(\eta e^{\ell k\eta})$ agrees with the Gr\"{o}nwall's inequality \citep{borkar1999strong} on the exponential dependence ($e^{k\eta}$ for effective time scaling $k\eta$). This undesirable exponential dependence motivates us to also present the non-asymptotic bound using a different notion --- the relative entropy via Girsanov formula in Lemma~\ref{lem:discrete-to-continuous}.
\end{remark}

\begin{lem}[Relative entropy]
	\label{lem:discrete-to-continuous}
	Under Assumptions~\ref{asmp:Lipschitz} and \ref{asmp:bound}, the relative entropy between stochastic processes $\{\xi_t, 0\leq t\leq k\eta\}$ in \eqref{eq:discr} and $\{\theta_t, 0\leq t\leq k\eta\}$ in \eqref{eq:conti} satisfies,
	\begin{align*}
		D_{\rm KL}\left( \mu(\theta_t, 0\leq t\leq k\eta) || \mu(\xi_t, 0\leq t\leq k\eta) \right)  \leq \left( \frac{\ell^2 M^2}{6} \beta \eta^3  +  \frac{\ell^2 p}{2} \eta^2 \right) \cdot k.
	\end{align*}
\end{lem}

\begin{remark}
	\rm
	Let's explain the pros and cons of the upper bound in Lemma~\ref{lem:discrete-to-continuous}. On the one hand, the bound on relative entropy reads $O((n+p)k\eta^2)$ when plug in $\beta = \frac{2n}{\eta}$, which results in better dependence on $k$ for the general non-convex case. On the other hand, the bound is not as desirable as in Lemma~\ref{lem:wass} for two reasons. First, notions like total variation or entropic distance can be very strong, which can be easily seen in the extreme case when $n \rightarrow \infty$ and $\beta = \frac{2n}{\eta} \rightarrow \infty$ --- the distribution of the discretized diffusion and the continuous time analog are two $\delta$ measures with total variation distance $1$, even though we know they are path-wise close. The Wasserstein distance captures the path-wise closeness. Second, bound in Lemma~\ref{lem:discrete-to-continuous} fails to provide more detailed characterization when the vector field $\mathbf{h}(x)$ enjoys the non-expansive property as in Lemma~\ref{lem:wass}.
\end{remark}

\section{Technical Proofs}
\label{sec:proof}

\subsection{Proof of Theorems}

In this section, we provide details of the proof of theorems in the paper.

\begin{proof}[Proof of Theorem~\ref{thm:couple.p}]
	The MasGrad updates can be represented as
	\begin{align}
		\theta_{t+1} = \theta_{t} - \eta \mathbf{V}(\theta_{t})^{-1} \mathbf{b}(\theta_t) + \sqrt{2\beta^{-1} \eta} \frac{\sum_{i=1}^n X_i(\theta_t)}{\sqrt{n}}.
	\end{align}
	Denote $S_n(X, \theta_t) = \frac{\sum_{i=1}^n X_i(\theta_t)}{\sqrt{n}}$.
	Under the Assumptions~(A.1) and (A.2), Thm. 6.1 in \cite{bobkov2013} (with $(4+\delta)$-moment condition) implies at each step $t$,
	\begin{align}
	D_{\rm KL}\left( \mu(S_n(X, \theta_t))|| \mu( \mathbf{g}_t ) | \theta_{t} \right)  = \frac{C}{n} + o\left( \frac{(\log n)^{\frac{p - (4+\delta)}{2}}}{n^{\frac{4+\delta-2}{2}}}  \right) = \frac{C}{n} + o\left( \frac{(\log n)^{\frac{p - (4+\delta)}{2}}}{n^{1+\frac{\delta}{2}}} \right),
	\end{align}
	conditioned on $\theta_t$, for some constant $C >0$.

	Apply the chain-rule for relative entropy, we know that
	\begin{align*}
		&D_{\rm KL}\left( \mu(\theta_t, t \in [T]) || \mu( \xi_t, t \in [T]) \right)  \\
		& = D_{\rm KL}\left( \mu(\theta_t, t \in [T-1]) || \mu( \xi_t, t \in [T-1]) \right) \\
		&\quad \quad + \int   D_{\rm KL}\left( \mu(S_n(X, \theta_{T-1}) || \mu(\mathbf{g}_{T-1}) | \theta_{T-1} \right) d \mu(\theta_t, t \in [T-1]) \\
		& \leq D_{\rm KL}\left( \mu(\theta_t, t \in [T-1]) || \mu( \xi_t, t \in [T-1]) \right) + \frac{C}{n} + o\left( \frac{(\log n)^{\frac{p - (4+\delta)}{2}}}{n^{1+\frac{\delta}{2}}} \right) \\
		& \leq D_{\rm KL}\left( \mu(\theta_t, t \in [T-2]) || \mu( \xi_t, t \in [T-2]) \right) \\
		&\quad + \int   D_{\rm KL}\left( \mu(S_n(X, \theta_{T-2}) || \mu(\mathbf{g}_{T-2}) | \theta_{T-2} \right) d \mu(\theta_t, t \in [T-2]) + o\left( \frac{(\log n)^{\frac{p - (4+\delta)}{2}}}{n^{1+\frac{\delta}{2}}} \right) \\
		& \leq \ldots \leq D_{\rm KL}\left( \mu(\theta_0)|| \mu(\xi_0)  \right) + \frac{CT}{n} + o\left( \frac{T(\log n)^{\frac{p - (4+\delta)}{2}}}{n^{1+\frac{\delta}{2}}} \right),
	\end{align*}
	where the second step uses the fact for $a, b>0$, $D_{\rm KL}(\mu(X)||\mu(Y)) = D_{\rm KL}(\mu(a+bX) || \mu(a+bY))$, therefore
	{\scriptsize
	\begin{align*}
		&D_{\rm KL}\left( \left. \mu\left(\theta_{T-1} - \eta \mathbf{V}(\theta_{T-1})^{-1} \mathbf{b}(\theta_{T-1}) + \sqrt{2\beta^{-1} \eta} \frac{\sum\limits_{i=1}^n X_i(\theta_{T-1})}{\sqrt{n}} \right) \left|\right|  \mu\left(\theta_{T-1} - \eta \mathbf{V}(\theta_{T-1})^{-1} \mathbf{b}(\theta_{T-1}) + \sqrt{2\beta^{-1} \eta} \mathbf{g}_{T-1} \right) \right| \theta_{T-1} \right) \\
		&= D_{\rm KL} \left( \mu\left(\sum_{i=1}^n X_i(\theta_{T-1})/\sqrt{n}\right) || \mu(\mathbf{g}_{T-1}) | \theta_{T-1} \right) = D_{\rm KL}\left( \mu(S_n(X, \theta_{T-1}) || \mu(\mathbf{g}_{T-1}) | \theta_{T-1} \right).
	\end{align*}
	}

	Apply the Pinsker's inequality that for any random variables $X, Y$, $$\frac{1}{2}D_{\rm TV}(\mu(X), \mu(Y))^2 \leq  D_{\rm KL}(\mu(X)||\mu(Y)),$$ we finish the proof.
\end{proof}

\begin{proof}[Proof of Theorem~\ref{thm:converge}]
	Proceed as in the proof of Lemma.~\ref{lem:conv.glm}, there exists $0\leq \tilde{c} \leq 1$ such that
	\begin{align*}
		& \E \left\{ L(\xi_{t+1}) | \xi_{t}\right\}\\
		&=  \E \left\{ L(\xi_t) +  \langle \mathbf{b}(\xi_t), \xi_{t+1} - \xi_{t} \rangle + \frac{1}{2} (\xi_{t+1} - \xi_{t})^T \mathbf{H}\left(\tilde{c} \xi_t + (1-\tilde{c}) \xi_{t+1} \right) (\xi_{t+1} - \xi_{t}) | \xi_{t} \right\}, \\
		&\leq  L(\xi_t) - \eta \langle \mathbf{b}(\xi_t), \mathbf{V}(\xi_t)^{-1} \mathbf{b}(\xi_t) \rangle + \E \left\{ \frac{\gamma}{2}  \| \eta \mathbf{V}(\xi_t)^{-1/2} \mathbf{b}(\xi_t) - \mathbf{V}(\xi_t)^{1/2} \sqrt{2\beta^{-1}\eta}\mathbf{g}_t \|^2 | \xi_t\right\}, \\
		& \leq L(\xi_t) - \eta \| \mathbf{V}(\xi_t)^{-1/2} \mathbf{b}(\xi_t) \|^2 + \frac{\eta^2 \gamma}{2} \| \mathbf{V}(\xi_t)^{-1/2} \mathbf{b}(\xi_t) \|^2  + \beta^{-1} \eta \gamma\E \left\{ \| \mathbf{V}(\xi_t)^{1/2} \mathbf{g}_t \|^2| \xi_t \right\}, \\
		& = L(\xi_t) - \frac{1}{2\gamma} \| \mathbf{V}(\xi_t)^{-1/2} \mathbf{b}(\xi_t) \|^2 + \beta^{-1} \langle \mathbf{I}_p,  \mathbf{V}(\xi_t)\rangle.
	\end{align*}
	Therefore we have
	\begin{align*}
		\E L(\xi_{t+1}) - \min_{\xi} L(\xi) &\leq (1 - \frac{\alpha}{\gamma}) (\E L(\xi_t) - \min_{\xi} L(\xi)) + \beta^{-1} \langle \mathbf{I}_p, \E \mathbf{V}(\xi_t) \rangle \\
		\E L(\xi_{t+1}) - \min_{\xi} L(\xi) &\leq (1 - \frac{\alpha}{\gamma}) (\E L(\xi_t) - \min_{\xi} L(\xi)) + \beta^{-1} p \cdot \max_{\xi} \| \mathbf{V}(\xi) \| \\
		\E L(\xi_{k}) - \min_{\xi} L(\xi) & \leq (1 - \frac{\alpha}{\gamma})^k (\E L(\xi_0) - \min_{\xi} L(\xi)) +  \frac{\beta^{-1} p \cdot \max_{\xi} \| \mathbf{V}(\xi) \|}{1 - (1 - \frac{\alpha}{\gamma})},
	\end{align*}
	one can make both the term on the right hand side to be bounded by $\epsilon/2$ by choosing
	\begin{align*}
		T = \frac{\gamma}{\alpha} \log \frac{2(L(\theta_0) - \min_{\theta} L(\theta))}{\epsilon} ~~\text{and}~~ n = \frac{4p \max_{\theta} \| \mathbf{V}(\theta) \|}{\alpha \epsilon}.
	\end{align*}
	We remind the reader that $\eta = 1/\gamma$ and $\beta = 2n/\eta$.
	We know $\xi_0 = \theta_0$. It is easily seen that the same argument holds with $\theta_t$ as the second moment of the Gaussian approximation using $\xi_{t+1}$ (conditioned on $\xi_t$) matches that of $\theta_{t+1}$ (conditioned on $\theta_t$). 
	
\end{proof}

\begin{proof}[Proof of Theorem~\ref{thm:acceleration}]
	
Let's denote for two matrices $A \preceq B$ denotes that $B-A$ is a positive semi-definite matrix.
First, observe that $\Cov[\beta(\mathbf{x}, w) \mathbf{x} ]   \preceq \E[\beta(\mathbf{x}, w)^2 \mathbf{x} \mathbf{x}^T ]$.
Using the bias and variance decomposition, one can further analytically evaluate,
\begin{align*}
	\mathbf{V}(w) =  \left( \mathbb{E}[ \xi(\mathbf{x})^2 \mathbf{x} \mathbf{x}^T] + \Cov[\beta(\mathbf{x}, w) \mathbf{x} ]  \right)^{1/2}, \quad \mathbf{H}(w) =  \E\left[ c''(\mathbf{x}^T w) \mathbf{x} \mathbf{x}^T  \right].
\end{align*}
Therefore, the following matrix inequalities hold
\begin{align*}
	\mathbb{E}[ \xi(\mathbf{x})^2 \mathbf{x} \mathbf{x}^T] & \preceq \mathbf{V}(w)^2 \preceq \mathbb{E}[ (\xi(\mathbf{x})^2 + \beta(\mathbf{x}, w)^2) \mathbf{x} \mathbf{x}^T].
\end{align*}

Under the condition that there exists $C>1$ such that $$C^{-1/3} < \frac{c''(x^T v)}{\xi(x)^2 + \beta(x, w)^2}  \leq \frac{c''(x^T v)}{\xi(x)^2} < C^{1/3},$$ then we have
\begin{align*}
	&\mathbf{H}(v) = \E\left[ c''(\mathbf{x}^T v) \mathbf{x} \mathbf{x}^T  \right] = \E\left[ \frac{c''(\mathbf{x}^T v)}{\xi(\mathbf{x})^2}  \xi(\mathbf{x})^2 \mathbf{x} \mathbf{x}^T  \right] \prec C^{1/3} \mathbf{V}(w)^2, \\
	&\mathbf{H}(v) = \E\left[ c''(\mathbf{x}^T v) \mathbf{x} \mathbf{x}^T  \right] = \E\left[ \frac{c''(\mathbf{x}^T v)}{\xi(\mathbf{x})^2 + \beta(\mathbf{x}, w)^2}  \left( \xi(\mathbf{x})^2 + \beta(\mathbf{x}, w)^2\right) \mathbf{x} \mathbf{x}^T  \right] \succ C^{-1/3} \mathbf{V}(w)^2.
\end{align*}

Let's recall the following facts that if $A \prec B$, then $\lambda_{\max}(A) < \lambda_{\max}(B)$ because take $v$ to be the top unit eigenvector of $A$,
$$
\lambda_{\max}(A) = v^T A v < v^T B v \leq \lambda_{\max}(B).
$$
Similarly, we have $\lambda_{\min}(A) < \lambda_{\min}(B)$. Also, if $A \prec B$, then for any symmetric matrix $S$, $S A S \prec S B S$.

Now because $\mathbf{H}(v) \prec C^{1/3} \mathbf{V}(w)^2$, take $w, v$ that maximize the LHS of the following
\begin{align*}
	\lambda_{\max} \left( [\mathbf{V}(w)]^{-1/2} \mathbf{H}(v) [\mathbf{V}(w)]^{-1/2} \right) &< C^{1/3} \lambda_{\max} \left( [\mathbf{V}(w)]^{-1/2} \mathbf{V}(w)^2 [\mathbf{V}(w)]^{-1/2} \right) \\
	& \leq C^{1/3} \max_w \lambda_{\max}(\mathbf{V}(w)).
\end{align*}
Similarly, because $C^{-1/3} \mathbf{V}(w)^2 \prec \mathbf{H}(v)$,
\begin{align*}
	\lambda_{\min} \left( [\mathbf{V}(w)]^{-1/2} \mathbf{H}(v) [\mathbf{V}(w)]^{-1/2} \right) &> C^{-1/3} \lambda_{\min} \left( [\mathbf{V}(w)]^{-1/2} \mathbf{V}(w)^2 [\mathbf{V}(w)]^{-1/2} \right) \\
	&\geq C^{-1/3} \min_w \lambda_{\min}(\mathbf{V}(w))
\end{align*}

Recall the definition of $\kappa_{\rm MasGrad}$, we know
\begin{align*}
	\kappa_{\rm MasGrad} &= \frac{\max_{w, v} \lambda_{\max} \left( [\mathbf{V}(w)]^{-1/2} \mathbf{H}(v) [\mathbf{V}(w)]^{-1/2} \right)}{\min_{w, v} \lambda_{\min} \left( [\mathbf{V}(w)]^{-1/2} \mathbf{H}(v) [\mathbf{V}(w)]^{-1/2} \right) }, \\
	&\leq \frac{C^{1/3} \max_w \lambda_{\max}(\mathbf{V}(w)) }{C^{-1/3} \min_w \lambda_{\min}(\mathbf{V}(w))} \\
	&\leq C^{2/3} \sqrt{ \frac{\max_w \lambda_{\max}(\mathbf{V}(w)^2)}{\min_w \lambda_{\min}(\mathbf{V}(w)^2)} } \leq C \sqrt{\frac{\max_v \lambda_{\max}(\mathbf{H}(v))}{\min_v \lambda_{\min}(\mathbf{H}(v))}} =  C \sqrt{\kappa_{\rm GD}}
\end{align*}
where the last step also uses the fact that
\begin{align*}
C^{-1/3} \mathbf{V}(w)^2 \prec \mathbf{H}(v) \prec C^{1/3} \mathbf{V}(w)^2.
\end{align*}

\end{proof}

\begin{proof}[Proof of Theorem~\ref{thm:non-convex}]

Denote $C \triangleq \max_{\theta} \| \mathbf{V}(\theta) \| $.
	Let's start with the the mean value theorem on the line segment between $\xi_{t+1}$ and $\xi_{t}$,
	\begin{align*}
		& \E \left\{ L(\xi_{t+1}) | \xi_{t}\right\}\\
		&=  \E \left\{ L(\xi_t) +  \langle \mathbf{b}(\xi_t), \xi_{t+1} - \xi_{t} \rangle + \frac{1}{2} (\xi_{t+1} - \xi_{t})^T \mathbf{H}\left(\tilde{c} \xi_t + (1-\tilde{c}) \xi_{t+1} \right) (\xi_{t+1} - \xi_{t}) | \xi_{t} \right\} \\
		&\leq L(\xi_t) - \eta \langle \mathbf{b}(\xi_t), \mathbf{V}(\xi_t)^{-1} \mathbf{b}(\xi_t) \rangle + \E \left\{ \frac{\gamma}{2}  \| \eta \mathbf{V}(\xi_t)^{-1/2} \mathbf{b}(\xi_t)- \sqrt{2\beta^{-1} \eta} \mathbf{V}(\xi_t)^{1/2} \mathbf{g}_t \|^2 | \xi_t\right\} \\
		&= L(\xi_t) - \left(\eta -\frac{\eta^2 \gamma}{2} \right) \| \mathbf{V}(\xi_t)^{-1/2} \mathbf{b}(\xi_t) \|^2  + \beta^{-1} \eta \gamma \E \| \mathbf{V}(\xi_t)^{1/2} \mathbf{g}_t \|^2 \\
		& \leq L(\xi_t) - \left(\eta -\frac{\eta^2 \gamma}{2} \right) \| \mathbf{V}(\xi_t)^{-1/2} \mathbf{b}(\xi_t) \|^2 +  C^{1/2} \cdot p  \beta^{-1} \eta \gamma.
	\end{align*}
Therefore, summing over $t \in [T]$, we have
	\begin{align*}
		L(\xi_0) - \min L(\theta) + C^{1/2} \cdot p  \beta^{-1} \eta \gamma T &\geq \sum_{t=0}^{T-1} \left(\eta -\frac{\eta^2 \gamma}{2} \right) \mathbb{E} \| \mathbf{V}(\xi_t)^{-1/2} \mathbf{b}(\xi_t) \|^2,  \\
		\E \min_{t \leq T} \| \mathbf{V}(\xi_t)^{-1/2} \mathbf{b}(\xi_t) \|^2 &\leq \frac{L(\theta_0) - \min L(\theta)+ C^{1/2} \cdot p  \beta^{-1} \eta \gamma T}{T \left(\eta -\frac{\eta^2 \gamma}{2} \right) }.
	\end{align*}
	Therefore we the choice $\eta = \frac{1}{\gamma}$, we have
	$$
	\E \min_{t \leq T} \| \mathbf{V}(\xi_t)^{-1/2} \mathbf{b}(\xi_t) \|^2 \leq \frac{2\gamma(L(\theta_0) - \min L(\theta))}{T} + C^{1/2} \cdot \frac{p}{n}.
	$$
	To obtain an $\epsilon$-stationary point in the sense that $\E \min_{t \leq T} \| \mathbf{b}(w_t) \| \leq \epsilon$, we need to
	\begin{align*}
		\frac{1}{C^{1/2}} \E \min_{t \leq T} \| \mathbf{b}(\xi_t) \|^2 \leq \E \min_{t \leq T} \| \mathbf{V}(\xi_t)^{-1/2} \mathbf{b}(\xi_t) \|^2 \leq \frac{2\gamma(L(\theta_0) - \min L(\theta))}{T} + C^{1/2} \cdot \frac{p}{n} \leq \frac{\epsilon^2}{C^{1/2}}.
	\end{align*}
	Hence, one can choose
	\begin{align*}
		T &=  \frac{ C^{1/2}\left[ 2\gamma(L(w_0) - \min L(w))  + C^{1/2} \cdot p\delta^2 \right]}{\epsilon^2}, \\
		n &= \frac{T}{\delta^2},
	\end{align*}
	to ensure
	$$
	\left(\E \min_{t \leq T} \| \mathbf{b}(w_t) \| \right)^2 \leq \E \min_{t \leq T} \| \mathbf{b}(w_t) \|^2 \leq \epsilon^2.$$
	And due to Thm.~\ref{thm:couple.p}, we know at the same time
	\begin{align*}
		D_{\rm TV}\left( \mu(\theta_t, t \in [T]), \mu( \xi_t, t \in [T] ) \right) \leq O(\sqrt{\frac{T}{n}}) = C \sqrt{\frac{T}{n}} = O_{\delta}(\delta).
	\end{align*}
	The total number of samples needed is $N = nT = O(\epsilon^{-4} \delta^{-2})$. Again, it is easy to see that the same argument holds with $\theta_t$ as the conditional second moment of $\xi_{t+1}$ matches that of $\theta_{t+1}$. 
\end{proof}

\subsection{Proof of Propositions}

This section dedicates to the proof of propositions.  

\begin{proof}[Proof of Proposition~\ref{thm:prox}]
Now let's analyze Moment Adjusted Proximal Gradient Descent in Eq.~\eqref{eq:mad-proximal}. For any $w$, and any $z \in \partial h(w_{t+1})$ in sub-gradient, following holds,
\begin{align}
	L(w_{t+1}) &= g(w_{t+1}) + h(w_{t+1}) \nonumber \\
	&\leq \left[ g(w_t) + \langle \nabla g(w_t), w_{t+1} - w_{t} \rangle + \frac{1}{2} \| w_{t+1} - w_{t} \|^2_{\mathbf{H}(\tilde{c})} \right] + \left[ h(w) + \langle z, w_{t+1} - w \rangle \right] \nonumber \\
	&\leq \left[ g(w) + \langle \nabla g(w_{t}), w_{t} - w \rangle - \frac{1}{2} \| w_t - w \|^2_{\mathbf{H}(c')} \right] \nonumber\\
	& \quad + \left[ \langle \nabla g(w_t), w_{t+1} - w_{t} \rangle + \frac{1}{2} \| w_{t+1} - w_{t} \|^2_{\mathbf{H}(\tilde{c})}  \right] + \left[ h(w) +  \langle z, w_{t+1} - w \rangle \right] \nonumber \\
	& = \left[ g(w) + \langle \nabla g(w_{t}), w_{t+1} - w \rangle + \frac{1}{2} \| w_{t+1} - w_{t} \|^2_{\mathbf{H}(\tilde{c})} - \frac{1}{2} \| w_t - w \|^2_{\mathbf{H}(c')} \right]   \nonumber \\
	& \quad + \left[ h(w) +  \langle z, w_{t+1} - w \rangle \right] \nonumber \\
	& = L(w) + \langle \nabla g(w_{t}) + z, w_{t+1} - w \rangle + \frac{1}{2} \| w_{t+1} - w_{t} \|^2_{\mathbf{H}(\tilde{c})} - \frac{1}{2} \| w_t - w\|^2_{\mathbf{H}(c')} \label{eq:proximal}.
\end{align}
Due to the optimality of the proximal updates in Eq.~\eqref{eq:mad-proximal}, we know
$$
0 \in \frac{1}{\eta} \mathbf{V}(w_{t+1} - w_t +\eta \mathbf{V}^{-1} \nabla g(w_t)) + \partial h(w_{t+1}),
$$
there exists $z \in \partial h(w_{t+1})$ such that
$$
\nabla g(w_t) + z = \frac{1}{\eta}  \mathbf{V} (w_{t} - w_{t+1}).
$$
Continue with Eq.~\eqref{eq:proximal}, and recall the definition of $\alpha, \gamma$,
one has
\begin{align}
	L(w_{t+1}) & \leq L(w) + \langle \frac{1}{\eta} \mathbf{V} (w_{t} - w_{t+1}), w_{t+1} - w \rangle + \frac{1}{2} \| w_{t+1} - w_{t} \|^2_{\mathbf{H}(\tilde{c})} - \frac{1}{2} \| w_t - w \|^2_{\mathbf{H}(c')} \nonumber\\
	& \leq L(w) + \langle \frac{1}{\eta} \mathbf{V} (w_{t} - w_{t+1}), w_{t+1} - w \rangle + \frac{\gamma}{2} \| w_{t+1} - w_{t} \|^2_{\mathbf{V}} - \frac{\alpha}{2} \| w_t - w \|^2_{\mathbf{V}}. \label{eqn:w_t}
\end{align}
Plug in $w = w_t$, we know if $\eta = \frac{1}{\gamma}$
$$
L(w_{t+1}) \leq L(w_t) - (\frac{1}{\eta} - \frac{\gamma}{2})\| w_{t+1} - w_{t} \|_{\mathbf{V}}^2 = L(w_t) - \frac{\gamma}{2} \| w_{t+1} - w_{t} \|_{\mathbf{V}}^2 \leq L(w_t).
$$
Plug in $w_* = \argmin L(w)$, one has
\begin{align}
	L(w_{t+1}) - L(w_*) &\leq \gamma \langle   w_{t}  - w_{t+1}, w_{t+1} - w_* \rangle_{\mathbf{V}} + \frac{\gamma}{2} \| w_{t+1} - w_{t} \|_{\mathbf{V}}^2 -  \frac{\alpha}{2} \| w_t - w_* \|^2_{\mathbf{V}} \nonumber\\
	& = - \frac{\gamma}{2} \| w_{t+1} - w_* \|_{\mathbf{V}}^2 + \frac{\gamma}{2} \| w_{t+1} - w_* + w_t - w_{t+1} \|_{\mathbf{V}}^2 -  \frac{\alpha}{2} \| w_t - w_* \|^2_{\mathbf{V}} \nonumber\\
	& = \frac{\gamma-\alpha}{2}  \| w_{t} - w_* \|_{\mathbf{V}}^2  -  \frac{\gamma}{2}\| w_{t+1} - w_* \|_{\mathbf{V}}^2 \nonumber \\
	\frac{2}{\gamma - \alpha} [L(w_{t+1}) - L(w_*)] & \leq \| w_{t} - w_* \|_{\mathbf{V}}^2 - \frac{\gamma}{\gamma - \alpha} \| w_{t+1} - w_* \|_{\mathbf{V}}^2 \label{eqn:w_star}
\end{align}
where the second equality follows due to opening the square
$$
-  \| w_{t+1} - w_* \|_{\mathbf{V}}^2 +  \| w_{t+1} - w_* + w_t - w_{t+1} \|_{\mathbf{V}}^2 = 2 \langle   w_{t}  - w_{t+1}, w_{t+1} - w_* \rangle_{\mathbf{V}} +  \| w_{t+1} - w_{t} \|_{\mathbf{V}}^2.
$$

Aggregating the above equations for $t = 0,\ldots T-1$ in a weighted way to form the telescoping sum, one has
\begin{align*}
	&\frac{2}{\alpha}\left[\left( \frac{\gamma}{\gamma - \alpha} \right)^T - 1 \right] (L(w_T) - L(w_*)) \\
  & \leq \frac{2}{\gamma - \alpha} \sum_{t=0}^{T-1} \left( \frac{\gamma}{\gamma - \alpha} \right)^t (L(w_{t+1}) - L(w_*))  \quad\text{by Eqn.~\eqref{eqn:w_t}}\\
  & \leq \sum_{t=0}^{T-1} \left( \frac{\gamma}{\gamma - \alpha} \right)^t  \left\{ \| w_{t} - w_* \|_{\mathbf{V}}^2 - \frac{\gamma}{\gamma - \alpha} \| w_{t+1} - w_* \|_{\mathbf{V}}^2 \right\} \quad\text{by Eqn.~\eqref{eqn:w_star}} \\
  & = \| w_0 - w_* \|^2_{\mathbf{V}} - \left( \frac{\gamma}{\gamma - \alpha} \right)^T \| w_T - w_* \|^2_{\mathbf{V}} \leq   \| w_0 - w_* \|^2_{
	\mathbf{V}}.
\end{align*}
Therefore we know if
\begin{align*}
	T \geq \frac{\gamma}{\alpha} \log \left( \frac{\alpha}{2\epsilon} \| w_0 - w_* \|^2_{
\mathbf{V}} + 1 \right),
\end{align*}
we have
$$
L(w_T) - L(w_*) \leq \epsilon.
$$
\end{proof}

\begin{proof}[Proof of Proposition~\ref{lem:self-norm-process}]
	Using standard definition of sample covariance matrix, and simple matrix algebra, we know
	\begin{align}
		(n-1)\widehat{\Sig} &=  \left( V_n^2 - n (\bar{\x} - \mu) (\bar{\x} - \mu)^T\right) \nonumber \\
		&=  V_n \left( I - n [V_n^{-1} (\bar{\x} - \mu)] \otimes [V_n^{-1} (\bar{\x} - \mu)]  \right)V_n \nonumber\\
		\widehat{\Sig}^{-1} &= (n-1) V_n^{-1} \left( I - \frac{1}{n} V_n^{-1} S_n S_n^T V_n^{-1} \right)^{-1} V_n^{-1} . \label{eq:root-verify}
	\end{align}
	Plug in the Woodbury identity (Lemma~\ref{lem:woodbury}) with the choice of $A = I, C = 1$ and $U = -V^T = \frac{1}{\sqrt{n}} V_n^{-1} S_n$, one has
	\begin{align*}
		& \left( I - \frac{1}{n} V_n^{-1} S_n S_n^T V_n^{-1} \right)^{-1} = I + \frac{ \frac{1}{n} V_n^{-1} S_n S_n^T V_n^{-1} }{1 - \frac{1}{n} S_n^T V_n^{-2} S_n}.
	\end{align*}
	Apply Lemma~\ref{lem:matrix-root}, the matrix root identity we derived, with $v = \frac{1}{\sqrt{n}} V_n^{-1} S_n$ and $c = 1/(1 -  \frac{1}{n} S_n^T V_n^{-2} S_n)$, we know
	\begin{align*}
		& I + \frac{ \frac{1}{n} V_n^{-1} S_n S_n^T V_n^{-1} }{1 - \frac{1}{n} S_n^T V_n^{-2} S_n} \\
		&= \left\{ I + \left(\frac{1}{\sqrt{1 - \frac{1}{n} S_n^T V_n^{-2} S_n}} - 1\right) \frac{1}{\frac{1}{n} S_n^T V_n^{-2} S_n} \frac{1}{n} V_n^{-1} S_n S_n^T V_n^{-1}  \right\}^2 \\
		&= \left\{ I + \left(\frac{1}{(1+\sqrt{1 - \frac{1}{n} S_n^T V_n^{-2} S_n})\sqrt{1 - \frac{1}{n} S_n^T V_n^{-2} S_n}}\right) \frac{1}{n} V_n^{-1} S_n S_n^T V_n^{-1}  \right\}^2. 
	\end{align*}
	Define
	\begin{align*}
		\widehat{\V}^{-1}  :=  \sqrt{n-1}  \left( I + \left(\frac{1}{(1+\sqrt{1 - \frac{1}{n} S_n^T V_n^{-2} S_n})\sqrt{1 - \frac{1}{n} S_n^T V_n^{-2} S_n}}\right) \frac{1}{n} V_n^{-1} S_n S_n^T V_n^{-1}  \right) V_n^{-1},
	\end{align*}
	then one can verify $\widehat{\V}$ is indeed root of $\widehat{\Sig}$ in the sense that $ \widehat{\V}\widehat{\V}^T = \widehat{\Sig}$, recalling \eqref{eq:root-verify}.
	Therefore, we know
	\begin{align*}
		&\sqrt{n} \widehat{\V}^{-1} (\bar{\x} - \mu) \\
		&= \sqrt{\frac{n-1}{n}} \left( I + \left(\frac{1}{(1+\sqrt{1 - \frac{1}{n} S_n^T V_n^{-2} S_n})\sqrt{1 - \frac{1}{n} S_n^T V_n^{-2} S_n}}\right) \frac{1}{n} V_n^{-1} S_n S_n^T V_n^{-1}  \right)  V_n^{-1} S_n \\
		&= V_n^{-1} S_n \cdot \sqrt{\frac{n-1}{n}} \left( 1 +  \frac{\frac{1}{n} S_n^T V_n^{-2} S_n}{(1+\sqrt{1 - \frac{1}{n} S_n^T V_n^{-2} S_n})\sqrt{1 - \frac{1}{n} S_n^T V_n^{-2} S_n}} \right) \\
		& = M_n \sqrt{\frac{n-1}{n - \| M_n \|^2}}.
	\end{align*}
\end{proof}

\begin{proof}[Proof of Proposition~\ref{lem:rank-one-mat-root-inv}]
	The proof relies on induction. Recall the definition of $\alpha_i$
	\begin{align}
		\label{eq:alpha_i}
		\alpha_i \triangleq (1 + \sqrt{1+v_{i+1}^T H_i^T H_i v_{i+1}})\sqrt{1+v_{i+1}^T H_i^T H_i v_{i+1}} \in \mathbb{R}.
	\end{align}
	Assume that
	\begin{align*}
		\left[ H_{i}^T H_{i} \right]^{-1} = I_d + \sum_{s=1}^{i} v_s v_s^T \triangleq \Sigma_i,
	\end{align*}
	we are going to show, for $i+1$, the induction holds. The following equations hold,
	\begin{align*}
		H_{i+1}^T H_{i+1} &= H_{i}^T H_{i} - \frac{2}{\alpha_i} H_{i}^T H_i v_{i+1} v_{i+1}^T H_i^T H_i + \frac{1}{\alpha_i^2} H_{i}^T H_i v_{i+1} v_{i+1}^T H_{i}^T  H_i v_{i+1} v_{i+1}^T H_i^T H_i \\
		& = H_{i}^T H_{i} - \frac{2}{\alpha_i} H_{i}^T H_i v_{i+1} v_{i+1}^T H_i^T H_i + \frac{v_{i+1}^T H_{i}^T  H_i v_{i+1}}{\alpha_i^2} H_{i}^T H_i v_{i+1}  v_{i+1}^T H_i^T H_i \\
		& = \Sigma_i^{-1} - \frac{2}{\alpha_i} \Sigma_i^{-1} v_{i+1} v_{i+1}^T \Sigma_i^{-1} + \frac{v_{i+1}^T\Sigma_i^{-1}v_{i+1}}{\alpha_i^2}  \Sigma_i^{-1} v_{i+1} v_{i+1}^T \Sigma_i^{-1} \\
		&= \Sigma_i^{-1} - \left( \frac{2}{\alpha_i} - \frac{v_{i+1}^T\Sigma_i^{-1}v_{i+1}}{\alpha_i^2} \right) \Sigma_i^{-1} v_{i+1} v_{i+1}^T \Sigma_i^{-1} \\
		&= \Sigma_i^{-1} - \frac{1}{1+v_{i+1}^T \Sigma_i^{-1}  v_{i+1}} \Sigma_i^{-1} v_{i+1} v_{i+1}^T \Sigma_i^{-1} \\
		&= (\Sigma_i + v_{i+1} v_{i+1}^T)^{-1},
	\end{align*}
	where the second last line uses \eqref{eq:alpha_i} and the following fact, 
	\begin{align*}
		\frac{2\alpha_i - v_{i+1}^T\Sigma_i^{-1}v_{i+1}}{\alpha_i^2} &= \frac{(1 + \sqrt{1+v_{i+1}^T H_i^T H_i v_{i+1}})^2}{\alpha_i^2} \\
	 	 & = \frac{1}{1+v_{i+1}^T \Sigma_i^{-1}  v_{i+1}},
	\end{align*}
	and the last line uses the Sherman-Morrison-Woodbury matrix identity, see Lemma~\ref{lem:woodbury}.
\end{proof}

\subsection{Proof of Lemmas}

We collect the supporting technical lemmas in this section.

\begin{lem}[Convergence: noiseless]
	\label{lem:conv.glm}
	Let $L(w): \mathbb{R}^p \rightarrow \mathbb{R}$ be a smooth convex function. Recall $\mathbf{b}(w) = \nabla L(w)$, and denote $\mathbf{H}(w)$ as the Hessian matrix of $L$. $\mathbf{V}(w) \in \mathbb{R}^{p \times p}$ is a positive definite matrix.
	Assume that
	\begin{align*}
		\alpha &\triangleq \min_{v, w}~ \lambda_{\min} \left( \mathbf{V}(w)^{-1/2} \mathbf{H}(v) \mathbf{V}(w)^{-1/2}  \right) >0, \\
		\gamma &\triangleq \max_{v, w}~ \lambda_{\max} \left( \mathbf{V}(w)^{-1/2} \mathbf{H}(v) \mathbf{V}(w)^{-1/2}  \right) >0.
	\end{align*}

	The deterministic updates $w_{t+1} = w_t - \eta \mathbf{V}(w_t)^{-1} \mathbf{b}(w_t),$ with step-size $\eta = 1/\gamma$, satisfies
	\begin{align*}
		L(w_{t+1}) - \min_{w} L(w) \leq \left( 1 - \frac{\alpha}{\gamma} \right) \left( L(w_t) - \min_{w} L(w) \right).
	\end{align*}
\end{lem}

\begin{remark}
	\rm
	If we define the condition number of MasGrad as
	\begin{align}
		\label{eq:cond.num}
		\kappa_{\rm MasGrad} = \frac{\max_{w, v} \lambda_{\max} \left( [\mathbf{V}(w)]^{-1/2} \mathbf{H}(v) [\mathbf{V}(w)]^{-1/2} \right)}{\min_{w, v} \lambda_{\min} \left( [\mathbf{V}(w)]^{-1/2} \mathbf{H}(v) [\mathbf{V}(w)]^{-1/2} \right) },\quad \kappa_{\rm GD} = \frac{\max_{v} \lambda_{\max} \left(  \mathbf{H}(v)  \right)}{\min_{v} \lambda_{\min} \left( \mathbf{H}(v) \right)},
	\end{align}
	compared to the condition number in gradient descent.
	To obtain a solution such that $L(w_t) - \min_{w} L(w) \leq \epsilon$, one need the number of iterations being
	$$
	t = \kappa_{\rm MasGrad} \cdot \log \frac{L(w_0) - \min_w L(w)}{\epsilon}.
	$$
\end{remark}

\begin{proof}[Proof of Lemma~\ref{lem:conv.glm}]
	First, let us focus on the line segment $\{ c w_t + (1-c) w_{t+1}, 0\leq c \leq 1\}$, by the mean value theorem, we know there exist a $\tilde{c} \in [0,1]$ such that the following holds
	\begin{align*}
		L(w_{t+1}) &=  L(w_t) + \langle \mathbf{b}(w_t), w_{t+1} - w_{t} \rangle + \frac{1}{2} (w_{t+1} - w_{t})^T \mathbf{H}\left(\tilde{c} w_t + (1-\tilde{c}) w_{t+1} \right) (w_{t+1} - w_{t}).
	\end{align*}
	Note $w_{t+1} = w_t - \eta \mathbf{V}(w_t)^{-1} \mathbf{b}(w_t)$, let's abbreviate $\mathbf{H}_{\tilde{c}}$ for the Hessian matrix at the middle point,
	\begin{align*}
		L(w_{t+1})&= L(w_t) - \eta \langle \mathbf{b}(w_t), \mathbf{V}(w_t)^{-1} \mathbf{b}(w_t) \rangle \\
		& \quad + \frac{\eta^2}{2} \left[ \mathbf{V}(w_t)^{-1/2} \mathbf{b}(w_t)\right]^T \mathbf{V}(w_t)^{-1/2} \mathbf{H}_{\tilde{c}} \mathbf{V}(w_t)^{-1/2} \left[ \mathbf{V}(w_t)^{-1/2} \mathbf{b}(w_t) \right], \\
		& \leq L(w_t) - \eta \| \mathbf{V}(w_t)^{-1/2} \mathbf{b}(w_t) \|^2 + \frac{\eta^2 \gamma}{2} \| \mathbf{V}(w_t)^{-1/2} \mathbf{b}(w_t) \|^2,  \\
		& = L(w_t) - \frac{1}{2\gamma} \| \mathbf{V}(w_t)^{-1/2} \mathbf{b}(w_t) \|^2.
	\end{align*}
	if we choose $\eta = \frac{1}{\gamma}$.

	For any $w$, on line segment $c w + (1-c) w_{t}$, we can use mean value theorem again,
	\begin{align*}
		&L(w) - L(w_t) \\
		&= \langle \mathbf{V}(w_t)^{-1/2} \mathbf{b}(w_t), \mathbf{V}(w_t)^{1/2} (w - w_{t}) \rangle + \frac{1}{2} (w - w_{t})^T \mathbf{H}\left(\tilde{c} w + (1-\tilde{c}) w_{t} \right) (w - w_{t}) \\
		& = \langle \mathbf{V}(w_t)^{-1/2} \mathbf{b}(w_t), \mathbf{V}(w_t)^{1/2} (w - w_{t}) \rangle \\
		& \quad + \frac{1}{2} \left[ \mathbf{V}(w_t)^{1/2}(w - w_{t}) \right]^T \mathbf{V}(w_t)^{-1/2}  \mathbf{H}_{\tilde{c}}  \mathbf{V}(w_t)^{-1/2} \left[ \mathbf{V}(w_t)^{1/2}(w - w_{t}) \right] \\
		& \geq \langle \mathbf{V}(w_t)^{-1/2} \mathbf{b}(w_t), \mathbf{V}(w_t)^{1/2} (w - w_{t}) \rangle  + \frac{\alpha}{2} \| \mathbf{V}(w_t)^{1/2}(w - w_{t}) \|^2 \\
		& \geq - \frac{1}{2\alpha} \| \mathbf{V}(w_t)^{-1/2} \mathbf{b}(w_t) \|^2.
	\end{align*}
	Therefore, choose $w$ that attains the minimum of $L$, combine the above two bounds, we know
	\begin{align*}
		L(w_{t+1}) - L(w_t) &\leq \frac{\alpha}{\gamma} (L(w) - L(w_t)), \\
		L(w_{t+1}) - L(w) &\leq (1 - \frac{\alpha}{\gamma}) (L(w_{t}) - L(w)).
	\end{align*}
\end{proof}

\begin{lem}
	\label{lem:consistency}
	Assume $p \precsim n/\log n$ and $\|  \mathbf{V}(\theta)^{-1} \mathbf{b}(\theta)  \| \leq C$ for some constant $C>0$. 
	Then there exists $\widehat{\V}(\theta)$, such that  
	\begin{align*}
		\widehat{\V}(\theta)^{-1} \widehat{\b}(\theta) - \mathbf{V}(\theta)^{-1} \mathbf{b}(\theta) = \overbrace{\widehat{\V}(\theta)^{-1} \left( \widehat{\b}(\theta) - \mathbf{b}(\theta) \right)}^{\text{self-normalized processes}} + E,
	\end{align*}
	where the $\ell_2$ norm of $E$ is upper bounded with probability at least $1- 2n^{-c}$,
	\begin{align*}
		\| E \| \leq  2C \sqrt{\frac{p}{n} \left(\log n + c^{-1} \log p\right) },
	\end{align*}
	for some constant $c>0$.
\end{lem}

\begin{proof}[Proof of Lemma~\ref{lem:consistency}]
	Start with standard decomposition
	\begin{align}
		\nonumber \widehat{\V}(\theta)^{-1} \widehat{\b}(\theta) - \mathbf{V}(\theta)^{-1} \mathbf{b}(\theta) &= \widehat{\V}(\theta)^{-1} \left( \widehat{\b}(\theta) - \mathbf{b}(\theta) \right) + \overbrace{\left( \widehat{\V}(\theta)^{-1}\V(\theta) - I  \right) \mathbf{V}(\theta)^{-1} \mathbf{b}(\theta)}^{\text{term $E$}}.
	\end{align}
	To bound the rate for term $E$, denote the singular value decomposition 
	\begin{align*}
		\mathbf{V}(\theta)^{-1} \widehat{\Sig}(\theta) \mathbf{V}(\theta)^{-1} = U \Lambda U^T
	\end{align*}
	Due to Theorem~41 in \cite{vershynin2010introduction} (random matrix with independent rows, isotropic, heavy-tail), one knows that with probability at least $1 - 2p \exp(-ct^2)$
	\begin{align*}
		(1 - t \sqrt{p/n})^2 I \precsim \Lambda \precsim (1 + t \sqrt{p/n})^2 I. 
	\end{align*}
	Let's define
	\begin{align*}
		\widehat{\V}(\theta) = \V(\theta) U \Lambda^{1/2} U^T
	\end{align*}
	it then is easy to verify $\widehat{\V}(\theta) \widehat{\V}(\theta)^T = \widehat{\Sig}(\theta)$ holds, meaning $\widehat{\V}(\theta)$ is a valid matrix root for $\widehat{\Sig}(\theta)$.
	In this case, one knows that term $E$ can be upper bounded
	\begin{align*}
		\left\| \left( \widehat{\V}(\theta)^{-1}\V(\theta) - I  \right) \mathbf{V}(\theta)^{-1} \mathbf{b}(\theta) \right\| &\leq \left\| \widehat{\V}(\theta)^{-1}\V(\theta) - I \right\|_{\rm op} \|\mathbf{V}(\theta)^{-1} \mathbf{b}(\theta) \| \\
		\text{by $UU^T= I$ and definition of $\widehat{\V}(\theta)$} \quad\quad & = \| U \Lambda^{-1/2} U^T - I \|_{\rm op} \|\mathbf{V}(\theta)^{-1} \mathbf{b}(\theta) \| \\
		& = \| \Lambda^{1/2} - I \|_{\rm op} \| \Lambda^{-1/2}\|_{\rm op} \|\mathbf{V}(\theta)^{-1} \mathbf{b}(\theta) \| \\
		& \leq \frac{C t \sqrt{p/n}}{1 - t \sqrt{p/n}},
	\end{align*}	
	with probability at least $1 - p \exp(-ct^2)$.
	Choose $t = \frac{\log p + c \log n}{c}$, we know under the condition 
	\begin{align*}
		\sqrt{\frac{p}{n} \left(\log n + c^{-1} \log p\right) } \leq \frac{1}{2},
	\end{align*}
	the following upper bound holds with probability $1- 2p \exp(-ct^2) = 1- 2n^{-c}$,
	\begin{align*}
		\| E \| \leq 2C \sqrt{\frac{p}{n} \left(\log n + c^{-1} \log p\right) } .
	\end{align*}
\end{proof}

\begin{lem}[Sherman-Morrison-Woodbury identity]
	\label{lem:woodbury}
	For matrices $A, U, C, V$ with matrices of the correct sizes, then 
	\begin{align*}
		(A + UCV)^{-1} = A^{-1} - A^{-1} U(C^{-1} + VA^{-1} U)^{-1} V A^{-1}.
	\end{align*}
\end{lem}

\begin{lem}[Matrix-root identity]
	\label{lem:matrix-root}
	For vector $V$, and scaler $c$, then
	\begin{align*}
		I + c v v^T = \left[ I + \frac{\sqrt{1+c\| v \|^2}-1}{\| v \|^2} v v^T \right]^2
	\end{align*}
\end{lem}
\begin{proof}[Proof of Lemma~\ref{lem:matrix-root}]
	The proof follows from the fact of openning the squares on the RHS,
	\begin{align*}
		\left[ I + \frac{\sqrt{1+c\| v \|^2}-1}{\| v \|^2} v v^T \right]^2 &= I + 2 \frac{\sqrt{1+c\| v \|^2}-1}{\| v \|^2} v v^T + \frac{(\sqrt{1+c\| v \|^2}-1)^2}{\| v \|^4} \| v \|^2 v v^T \\
		&= I + c vv^T.
	\end{align*}
\end{proof}

\begin{proof}[Proof of Lemma~\ref{lem:wass}]
	The proof is motivated from \citep{dalalyan2017further}. We will show that the proof extends to more general vector fields $\mathbf{h}$ using the notion of expansiveness \citep{hardt2015train}, without requiring $\mathbf{h}$ to be the gradient of a strongly convex function. Another difference is that we are tracking the difference between the Cauchy discretization $\xi_t$ and the Langevin diffusion $\theta_t$, instead of characterizing the distance of $\xi_t$ to the invariant measure. In addition, we generalize the proof to review the explicit dependence on the inverse temperature $\beta$. 

	Consider $\theta_t$ and $\xi_t$ defined using the same Brownian motion $B_t$, then we have
	\begin{align*}
		\| \xi_{k\eta} - \theta_{k\eta} \| &\leq  \| [\xi_{(k-1)\eta} - \eta \mathbf{h}(\xi_{(k-1)\eta})] - [\theta_{(k-1)\eta} - \eta \mathbf{h}(\theta_{(k-1)\eta})]  \| \\
		& \quad \quad + \left\| \int_{(k-1)\eta}^{k\eta} \left[ \mathbf{h}(\theta_t) - \mathbf{h}(\theta_{(k-1)\eta}) \right] dt \right\|  \\
		\left( \E \| \xi_{k\eta} - \theta_{k\eta} \|^2 \right)^{1/2} &\leq \left( \E \| [\xi_{(k-1)\eta} - \eta \mathbf{h}(\xi_{(k-1)\eta})] - [\theta_{(k-1)\eta} - \eta \mathbf{h}(\theta_{(k-1)\eta})] \|^2 \right)^{1/2} \\
		& \quad \quad + \underbrace{\left( \E \left\|\int_{(k-1)\eta}^{k\eta} \left[ \mathbf{h}(\theta_t) - \mathbf{h}(\theta_{(k-1)\eta}) \right] dt \right\|^2 \right)^{1/2}}_\text{defined as $\Delta$} \\
		& \leq \delta \E \left( \| \xi_{(k-1)\eta} - \theta_{(k-1)\eta} \|^2 \right)^{1/2} + \Delta
	\end{align*}
	where the first two steps use triangle inequality, on $\mathbb{R}^p$ and $\ell_2$ space associated with $\E$ respectively. The last step uses the following
	fact about the $\delta$-expansiveness,
	\begin{align*}
		\| [\xi_{(k-1)\eta} - \eta \mathbf{h}(\xi_{(k-1)\eta})] - [\theta_{(k-1)\eta} - \eta \mathbf{h}(\theta_{(k-1)\eta})]  \| &\leq \delta \| \xi_{(k-1)\eta} - \theta_{(k-1)\eta} \|.
	\end{align*}
	For the term $\Delta$,
	\begin{align*}
		\Delta^2 &= \E \sum_{i=1}^d \left| \int_{(k-1)\eta}^{k\eta} \left[ \mathbf{h}(\theta_t) - \mathbf{h}(\theta_{(k-1)\eta}) \right]_i dt \right|^2 \\
		&\leq \E \sum_{i=1}^d  \eta  \int_{(k-1)\eta}^{k\eta} |\left[ \mathbf{h}(\theta_t) - \mathbf{h}(\theta_{(k-1)\eta}) \right]_i |^2 dt  \quad \text{by Cauchy-Schwartz} \\
		&= \eta \int_{(k-1)\eta}^{k\eta} \E \| \mathbf{h}(\theta_t) - \mathbf{h}(\theta_{(k-1)\eta})  \|^2 dt \\
		&\leq \eta \ell^2 \int_{(k-1)\eta}^{k\eta} \E \| \theta_t - \theta_{(k-1)\eta}  \|^2 dt \quad \text{by $\ell$-Lipschitz} \\
		&= \eta \ell^2 \int_{(k-1)\eta}^{k\eta} \E \left\| - \int_{(k-1)\eta}^t \mathbf{h}(\theta_{s}) ds + \sqrt{2\beta^{-1}}  (B_t - B_{(k-1)\eta})  \right\|^2 dt.
	\end{align*}
	\begin{align*}
		& \leq \eta \ell^2 \int_{(k-1)\eta}^{k\eta} \left\{ 2 \E \left\| - \int_{(k-1)\eta}^t \mathbf{h}(\theta_{s}) ds \right\|^2 + 2 \E \left\| \sqrt{2\beta^{-1}}  (B_t - B_{(k-1)\eta})  \right\|^2  \right\}dt \\
		& \leq 2\eta \ell^2 \int_{(k-1)\eta}^{k\eta} (t - (k-1)\eta) \int_{(k-1)\eta}^t \E\| h(\theta_s) \|^2 ds dt + 2\eta \ell^2 \int_{(k-1)\eta}^{k\eta} 2\beta^{-1} p (t-(k-1)\eta) dt \\
		& \leq 2\eta \ell^2 \int_{(k-1)\eta}^{k\eta} (t - (k-1)\eta)^2 M^2 dt + 2 \ell^2 p \beta^{-1} \eta^3 \quad \text{by $M$-boundedness}\\
		& \leq \frac{2}{3} \ell^2 M^2 \eta^4 + 2\ell^2 p \beta^{-1} \eta^3.
	\end{align*}
	Then going back to the original equation we are trying to bound
	\begin{align}
		\left( \E \| \xi_{k\eta} - \theta_{k\eta} \|^2 \right)^{1/2} \leq \left(\frac{2}{3} \ell^2 M^2 \eta^4 + 2\ell^2 p \beta^{-1} \eta^3 \right)^{1/2} \cdot \sum_{i=0}^{k-1} \delta^i.
	\end{align}
\end{proof}

\begin{proof}[Proof of Lemma~\ref{lem:discrete-to-continuous}]
	The proof follows from calculations as in \cite{dalalyan2017theoretical,raginsky2017non}.
	The continuous-time interpolation enjoys the same distribution as $\xi_{k\eta}$ for all $k$.
	One can apply Girsanov formula to calculate the relative entropy
	\begin{align*}
		& D_{\rm KL}\left( \mu(\theta_t, 0\leq t\leq k\eta) || \mu(\xi_t, 0\leq t\leq k\eta) \right) \\
		& =  \frac{\beta}{4} \int_0^{k \eta} \mathbb{E}\| \mathbf{h}(\xi_t) - \mathbf{h}(\xi_{\lfloor t/\eta \rfloor \eta})  \|^2 dt \\
		& = \frac{\beta}{4} \sum_{i=0}^{k-1}\int_{i\eta}^{(i+1)\eta} \mathbb{E}\| \mathbf{h}(\xi_t) - \mathbf{h}(\xi_{i\eta})  \|^2 dt \\
		& \leq \frac{\ell^2 \beta}{4} \sum_{i=0}^{k-1}\int_{i\eta}^{(i+1)\eta} \mathbb{E}\| \xi_t - \xi_{i\eta}  \|^2 dt \\
		& = \frac{\ell^2 \beta}{4} \sum_{i=0}^{k-1}\int_{i\eta}^{(i+1)\eta} \mathbb{E}\| - (t-i\eta) \mathbf{h}(\xi_{i \eta}) + \sqrt{2\beta^{-1}} (B_t - B_{i\eta})  \|^2 dt \\
		& \leq \frac{\ell^2 \beta}{4} \sum_{i=0}^{k-1}\int_{i\eta}^{(i+1)\eta} \left[ 2 (t-i\eta)^2 \E\| \mathbf{h}(\xi_{i \eta}) \|^2 +  p \cdot 4\beta^{-1} (t-i \eta) \right] dt \\
		& = \frac{\ell^2 \beta}{4} \left[ \frac{2}{3} \eta^3 \sum_{i=0}^{k-1} \E\| \mathbf{h}(\xi_{i \eta}) \|^2  + k \cdot 2p\beta^{-1} \eta^2 \right] \\
		& = \frac{\ell^2}{6} \beta \eta^3 \sum_{i=0}^{k-1} \E\| \mathbf{h}(\xi_{i \eta}) \|^2 +  \frac{\ell^2 p}{2} k\eta^2
	\end{align*}
	Now recall that $\mathbf{h}$ is $M$-bounded, therefore, we know,
	\begin{align*}
		D_{\rm KL}\left( \mu(\theta_t, 0\leq t\leq k\eta) || \mu(\xi_t, 0\leq t\leq k\eta) \right) \leq \left( \frac{\ell^2 M^2}{6} \beta \eta^3  +  \frac{\ell^2 p}{2} \eta^2 \right) \cdot k.
	\end{align*}
\end{proof}

\section{Further details of the experiment}
\label{sec:exp-details}

\paragraph{Linear model.}

Let us provide the full details of the experiment. In the experiment, we generate a larger number of samples as the population (so that we can evaluate $\mathbf{V}$ easily), then use bootstrap to sample from this population at each step. The population minimizer can be solved using least squares. Here each row of the ``population'' data matrix $X \in \mathbb{R}^{500 \times 4}$ is sampled from a multivariate Gaussian independently, with a covariance matrix $\Sigma$ that has condition number $30.98$. Each step we independently subsample $n = 50$ rows with replacement. The response is generated from a well-specified linear model with additive standard Gaussian noise. The step-size is through calculating the smoothness parameter $\gamma$ as in Thm.~\ref{thm:acceleration}.

\paragraph{Logistic model.}

Again we will provide the full details of the experiment.
We fix a step-size $\eta = 0.2$ (other step-sizes essentially provide similar results), which implies the inverse temperature is $\beta = 2n/\eta = 250$. The data matrix $X \in \mathbb{R}^{500 \times 4}$ is generated from multivariate Gaussian with identity covariance. The response is generated from a well-specified logistic model with each coordinate of $w_*$ uniformly sampled between $[1,2]$.

\paragraph{Gaussian mixture.}

The likelihood for a data point $z$ is
\begin{align*}
	\ell(\theta;z) = - \log \left( \sum_{i=1}^p q_i \phi(z - \theta_i) \right), \quad \text{s.t.}~\sum_{i=1}^p q_i=1,
\end{align*}
where $\phi(x) = \frac{1}{\sqrt{2\pi}\sigma}e^{-\frac{x^2}{2\sigma^2}}$ denotes the density function for Gaussian.
Here in simulations we consider the case when the mixture probability $q_i, i\in [p]$ is known and uniform for the simplicity that we can apply the MasGrad without equality constraints\footnote{When the mixture probability is also unknown, one will need to consider adding a proper barrier function before applying the gradient method.}, and we have a clear picture of the global optima due to symmetry.

\paragraph{Shallow neural networks.}

In the experiment we generate from a well-specified model with very small additive Gaussian noise. However, due to the presense of the hidden layer, the problem is non-convex with many local optima. To break the ReLU scaling invariance (i.e., $\{c W_1, 1/c W_2\}$ is equivalent to $\{W_1, W_2\}$, for the purpose of letting stationary points more separable), we add a non-programmable constant in each layer in the experiment, namely $f_w(x) = \sigma(1+ W_2 \sigma(\mathbf{1}+ W_1 x))$. Because we generate the data from a well specified model, we also present the true parameter in the plot. Here we choose $n = 30$, and each step we subsample with replacement from $N = 300$ data points. The step-size is fixed to be $\eta = 0.1$, which implies the inverse temperature being $\beta = 600$.

\end{document}